\documentclass[lettersize,journal]{IEEEtran}
\usepackage{amsmath,amsfonts}
\usepackage{algorithmic}
\usepackage{algorithm}
\usepackage{array}
\usepackage[caption=false,font=normalsize,labelfont=sf,textfont=sf]{subfig}
\usepackage{textcomp}
\usepackage{stfloats}
\usepackage{url}
\usepackage{verbatim}
\usepackage{graphicx}
\usepackage{cite}
\usepackage{amsfonts,amssymb}
\usepackage{amsthm}
\usepackage{amssymb}
\usepackage{booktabs}
\def\DEFEQ{\overset{\underset{\mathrm{\Delta}}{}}{=}}
\def\INNERDOT#1{\langle #1 \rangle}
\def\Q1{\tilde{q}_1}
\def\Z2{\tilde{z}_2}

\newtheorem{proposition}{Proposition}

\begin{document}

\title{Learning the Relation between Similarity Loss and Clustering Loss in Self-Supervised Learning}

\author{Jidong~Ge,~
		Yuxiang~Liu,~
		Jie~Gui,~Senior~Member,~IEEE,~
		Lanting~Fang,
		Ming~Lin,
		James Tin-Yau Kwok,~Fellow,~IEEE,~
		LiGuo~Huang,
		Bin~Luo
\thanks{J. Ge, Y. Liu, and B. Luo are with
State Key Laboratory for Novel Software Technology, Software Institute, Nanjing University, 
Nanjing 210093, China.
(e-mail: gjd@nju.edu.cn, mg1932010@smail.nju.edu.cn, luobin@nju.edu.cn).}
\thanks{J. Gui and L. Fang are with the School of Cyber Science and Engineering, Southeast University and Purple Mountain Laboratories, Nanjing 210000, China (e-mail: \{guijie, lanting\}@seu.edu.cn).}
\thanks{M. Lin is with Amazon.com L.L.C. and this work was done at Alibaba Group (e-mail: minglamz@amazon.com).}
\thanks{James Tin-Yau Kwok is with the Department of Computer Science and Engineering,
The Hong Kong University of Science and Technology, Hong Kong
999077, China (e-mail: jamesk@cse.ust.hk).}
\thanks{L. Huang is with the Department of Computer Science, Southern Methodist University, Dallas, TX 75205, USA (email:  lghuang@smu.edu).}
}

\markboth{Journal of \LaTeX\ Class Files,~Vol.~14, No.~8, August~2021}%
{Shell \MakeLowercase{\textit{et al.}}: A Sample Article Using IEEEtran.cls for IEEE Journals}


\maketitle

\begin{abstract}
Self-supervised learning enables networks to learn discriminative features from massive data itself. Most state-of-the-art methods maximize the similarity between two augmentations of one image based on contrastive learning. By utilizing the consistency of two augmentations, the burden of manual annotations can be freed. Contrastive learning exploits instance-level information to learn robust features. However, the learned information is probably confined to different views of the same instance. In this paper, we attempt to leverage the similarity between two distinct images to boost representation in self-supervised learning. In contrast to instance-level information, the similarity between two distinct images may provide more useful information. Besides, we analyze the relation between similarity loss and feature-level cross-entropy loss. These two losses are essential for most deep learning methods. However, the relation between these two losses is not clear. Similarity loss helps obtain instance-level representation, while feature-level cross-entropy loss helps mine the similarity between two distinct images. We provide theoretical analyses and experiments to show that a suitable combination of these two losses can get state-of-the-art results. Code is available at https://github.com/guijiejie/ICCL.
\end{abstract}

\begin{IEEEkeywords}
Self-supervised learning, Image representation, Image classification.
\end{IEEEkeywords}

\section{Introduction}\label{sec:introduction}

\IEEEPARstart{R}{ecently}, un-/self-supervised representation learning has made steady progresses. Many self-supervised methods~\cite{hjelm2018learning,oord2018representation,bachman2019learning,caron2019unsupervised,S4L,MoCo,gpt3,SimCLR,henaff2020data,barlowtwins,dino, MAE} are closing the performance gap with supervised pretraining in computer vision. These methods leverage the property of the data itself. Most self-supervised methods attempt to build upon the instance discrimination~\cite{bojanowski2017unsupervised,dosovitskiy2015discriminative,wu2018unsupervised,cao2020parametric} task by maximizing the agreement between two augmentations of one image and scatter different instances. The encouraging results of self-supervised learning depend on strong transformations~\cite{hjelm2018learning,SimCLR,DBLP:conf/nips/Tian0PKSI20} (e.g., image crop and color distortion) and similarity loss. BYOL~\cite{BYOL} and SimSiam~\cite{SimSiam} extend similarity loss and remove the dependency on negative instances~\cite{richemond2020byol,tian2020understanding,directpred}. These methods implicitly do scattering and learn robust representations of different transformations of the same instance. In this paper, contrastive learning based methods represent methods such as MoCo~\cite{MoCov3} and BYOL. The key point of those methods is to minimize the similarity between augmentations. 

Unlike contrastive learning based methods that learn invariance to transformations~\cite{PIRL}, some works attempt to utilize clustering~\cite{yang2016joint,xie2016unsupervised,yan2020clusterfit,huang2019unsupervised,caron2018deep,DBLP:conf/iclr/AsanoRV20a,SwAV,li2021contrastive} with pseudo-labels. Most instance-level contrastive learning based methods may suffer from the misleading of similar backgrounds~\cite{zhao2020distilling}. No matter which transformation we choose, the image background may not be discarded entirely. The background pixels provide a shortcut to minimize similarity loss. By contrast, the similarity between distinct images may improve the robustness of background information. Images of the same object in different backgrounds are learned to maximize the similarity. This learning manner is pivotal and more similar to the learning manner of human beings. People can ignore the background because they have already seen hundreds of thousands of the same object in different backgrounds. Traditional clustering-based methods classify images through pseudo-labels. Those methods may correlate images of the same class. However, the generation of pseudo-labels needs much computation. Some online clustering methods~\cite{SwAV} assign labels for batch examples by Sinkhorn-Knopp algorithm~\cite{cuturi2013sinkhorn}. Sinkhorn-Knopp algorithm assures that batch examples are equally partitioned by the prototypes, preventing the trivial solution where every image has the same label.

Mining the similarity between two distinct images is a possible manner to improve discrimination. Most of the state-of-the-art methods (e.g., SwAV~\cite{SwAV} and DINO~\cite{dino}) directly leverage feature-level cross-entropy loss to learn the similarity between different images. In general, similarity loss and feature-level cross-entropy loss are used in different styles of self-supervised methods. In SimSiam, authors find cross-entropy loss may not be applicable to contrastive learning based methods. In this paper, we try to analyze the relation between these two losses. Through theoretical and experimental analyses, we point out that these two losses can be complementary. From the perspective of gradients, we analyze the difference between similarity loss and cross-entropy loss. These findings imply that a suitable combination may boost class-level and instance-level representations. Our contributions are listed as follows.
\begin{itemize}
	\item We demonstrate that supervised learning can catch the relation between different images of the same class. Besides, it is feasible for self-supervised methods to leverage the similarity between two distinct images. 
	\item We provide theoretical and experimental analyses to explain why previous contrastive learning based methods (e.g., SimSiam) have inferior results with cross-entropy loss. This point is critical to maintaining the advantages of similarity loss and cross-entropy loss.
	\item In contrast to those clustering-based methods, we focus on the relation between similarity loss and feature-level cross-entropy loss. Based on this relation, we propose a simple but effective method to exploit both instance-level contrastive and intra-class contrastive learning (\textbf{ICCL}). Our method attempts to mine the information between two distinct images in a suitable way, which reduces the impact of wrong clustering. Compared with SwAV and DINO, our method can work without centering (used in DINO) and Sinkhorn-Knopp (used in SwAV). The hyper-parameter settings are robust to different datasets. 
\end{itemize}

\section{Prelimilaries}
The intention of this section is to introduce some notations of different loss functions in this paper. In particular, we provide formal definitions for loss functions in SimSiam/BYOL (called similarity loss) and loss functions in SwAV/DINO (called clustering/cross-entropy loss).

\subsection{Methods Based on Similarity Loss}
Many methods use similarity loss~\cite{hadsell2006dimensionality} to maximize the agreement of two views of the same image. Generally speaking, the loss function in MoCo is a typical similarity loss function
\begin{equation}
	L_{contrastive}^{q} = - \log{\frac{\INNERDOT{q, k_{+}} / \tau}{\sum_{i=0}^{|B|} \INNERDOT{q, k_i} / \tau}}.
\end{equation}
Here $q$ is the representation of an image. $L_{contrastive}^{q}$ denotes the loss for representation $q$. As most self-supervised learning loss functions can be divided into the sum of losses corresponding to a single instance, we will omit the superscript of loss function in the following. $k_{+}$ denotes the positive example in the batch. Traditionally, two transformations $T_1$ and $T_2$ will transform image $I$ to different views of the same image. The representations of these views will be regarded as positives. $B$ denotes the batch of data and $|B|$ denotes the batch size. The similarity loss intends to make the similarity between positive features large and reduce the similarity between negative features. In \cite{DBLP:conf/icml/0001I20}, authors provide a detail deconstruction for similarity loss
\begin{align}
	L_{contrastive} &= - \log{\frac{\INNERDOT{q, k_{+}} / \tau}{\sum_{i=0}^{|B|} \INNERDOT{q, k_i} / \tau}} \nonumber \\
	&= \log{\sum_{i=0}^{|B|} \exp{\INNERDOT{q, k_i} / \tau}} - \frac{\INNERDOT{q, k_{+}}}{\tau}.
\end{align}
The first term is called uniformity term. If $q$ and $k+$ are normalized to unit, the $\INNERDOT{q, k_+}$ in the second term is the cosine similarity. Therefore, the loss function for MoCo may also be regarded as cosine similarity loss with uniformity term. 

For similarity loss, we define the output features computed by the neural network $z'_1$ and $z'_2$, respectively. $z'_1$ and $z'_2$ are two views of the same image. In BYOL~\cite{BYOL} and SimSiam~\cite{SimSiam}, one of the features will be passed through an extra predictor (e.g., the predictor encodes $z'_1$ as $q_1$). The ultimate features are denoted as $q_1$ and $z_2$. It should be noted that $z_2$ will not pass the gradients to the network ($z_2 = sg(z'_2)$ and $sg()$ denotes the stop gradient operation). With the definition of $\tilde{q} \DEFEQ \frac{q}{|| q ||}$, where $||.||$ denotes the $l_2$ norm, similarity loss can be represented by
\begin{equation}
	L_{similarity} = - \INNERDOT{\Q1, \Z2}.
	\label{equ:loss_contrastive}
\end{equation}
Here $\langle . , . \rangle$ is inner product. The uniformity term~\cite{SimCLR, chen2020big, MoCo, MoCov3} is ignored. This loss will maximize the similarity between two views of the same image and learn useful representations that are robust to strong augmentations. Contrastive learning (include methods use $L_{contrastive}$ and $L_{similarity}$) based methods are robust to various scales of datasets. 

\subsection{Methods Based on Clustering Loss}
Another form of self-supervised learning are based on clustering~\cite{caron2018deep, DBLP:conf/iclr/AsanoRV20a, SwAV, dino}. These methods generate pseudo-labels and use pseudo-labels to maximize cross-entropy loss:
\begin{align}
	L_{ce} = - \sum_{i=1}^{D} f(i, z_2) \log{p(i|q_1, \tau)}, \nonumber\\
	where \quad p(i|x, \tau) = \frac{\exp(\frac{x^{(i)}}{\tau})}{\sum_{j=1} \exp(\frac{x^{(j)}}{\tau})}.
	\label{equ:traditional_ce}
\end{align}
Here $f$ is the function to generate pseudo-labels (e.g., $f$ is Sinkhorn-Knopp algorithm in SwAV and $p(i|z_2)$ with centering mechanism in DINO). $C$ is the dimensionality of vectors $q_1$ and $z_2$. $W$ denotes clustering prototypes, which is a $C$-by-$D$ matrix. $\tau$ is the temperature~\cite{hinton2015distilling} to adjust the sharpness of the probability distribution, and the default value is 1. These clustering-based methods attempt to exploit the relation of different instances to learn robust representations. However, these methods may suffer from incorrect clustering. The learned representation is based on pseudo-labels. The hyper-parameters may be hard to be extended to a large number of datasets.

\section{Method}
In this section, we first illustrate the difference between supervised learning and self-supervised learning from recall metric. The difference indicates the distribution of intra-class data points is dispersed for self-supervised methods. However, supervised learning may aggregate intra-class features and expand the distance between different categories. This point motivates our following study. Then we analyze why $L_{ce}$ can capture the similarity between two distinct images. Then we demonstrate how to establish the relation between similarity loss and cross-entropy loss. Finally, we describe the details of our method and the difference from other clustering-based methods.

\subsection{The Intra-Class Distance for Self-Supervised Learning}
\begin{table}[!t]
    \caption{The comparison of supervised methods and self-supervised methods. The default network backbone is ResNet-50. The default training epochs are 100.}
    \label{tab:recall_ans}
    \centering
    \begin{tabular}{c|cc}
        \toprule
        & \multicolumn{2}{c}{ImageNet} \\
        \cline{2-3}
        Methods & $Precision@k=5$ & top-1 acc\\
        \midrule
        Supervised & \textbf{52.1} & \textbf{76.5}\\
        SimSiam & 27.1 & 67.3\\
        BarlowTwins & 28.2 & 67.4 \\
        BYOL & 28.1 & 66.0 \\
        BYOL-300ep & 32.4 & \underline{72.2} \\
        MoCo & 27.5 & 67.4 \\
        MoCo-ViT & \underline{37.0} & 69.1
    \end{tabular}
\end{table}
Although self-supervised learning gets promising results in many tasks and datasets, it still has several problems. For example, most self-supervised learning methods should leverage the linear evaluation protocol for image classification tasks. The trainable fully-connected layer is used to distinguish the features of different classes. This fully-connected layer is essential as Table~\ref{tab:recall_ans} shows. We use $Precision@k = num(true)/k$ to express the recall metric of different self-supervised methods. $num(true)$ denotes the number of positive nearest neighbors of the query image in the returned $k$ images. The positive examples are defined as images of the same class. This metric denotes the ability to recall images of the same category. In other words, if the features' distribution of the same class is compact and the centroid distance of different categories is relatively large, the $Precision@k$ may have a good performance. The outputs of backbone are used as the retrieval features. For top-1 accuracy, we train an extra fully-connected layer to do classification.

As Table~\ref{tab:recall_ans} shows, self-supervised learning methods are hard to learn compact representations from views of different instances. Although the gap of top-1 accuracy is close, we find the recall metric of self-supervised learning methods is still far less than supervised learning. This point indicates that most self-supervised learning may learn rough representations. 

Fig.~\ref{fig:distribution} visualizes the results of the network in different stages. Self-supervised learning is hard to capture intra-class relation at the beginning of the training because centroids of different classes are close. For example, SwAV will try to do clustering at the beginning of the training. However, it is hard to learn useful clustering information from massive irrelevant data. By contrast, if one can do clustering in Fig.~\ref{fig:short-b}, the clustering information may help networks to compact the features of the same class. One of the approaches is to avoid using intra-class information at the beginning of the training and leverage the intra-class information after several epochs. Therefore, the key point is to find a loss function that can directly replace the instance-level loss function from the perspective of the gradient.

\begin{figure}[!t]
\centering
\subfloat[]{
    \includegraphics[width=1.6in]{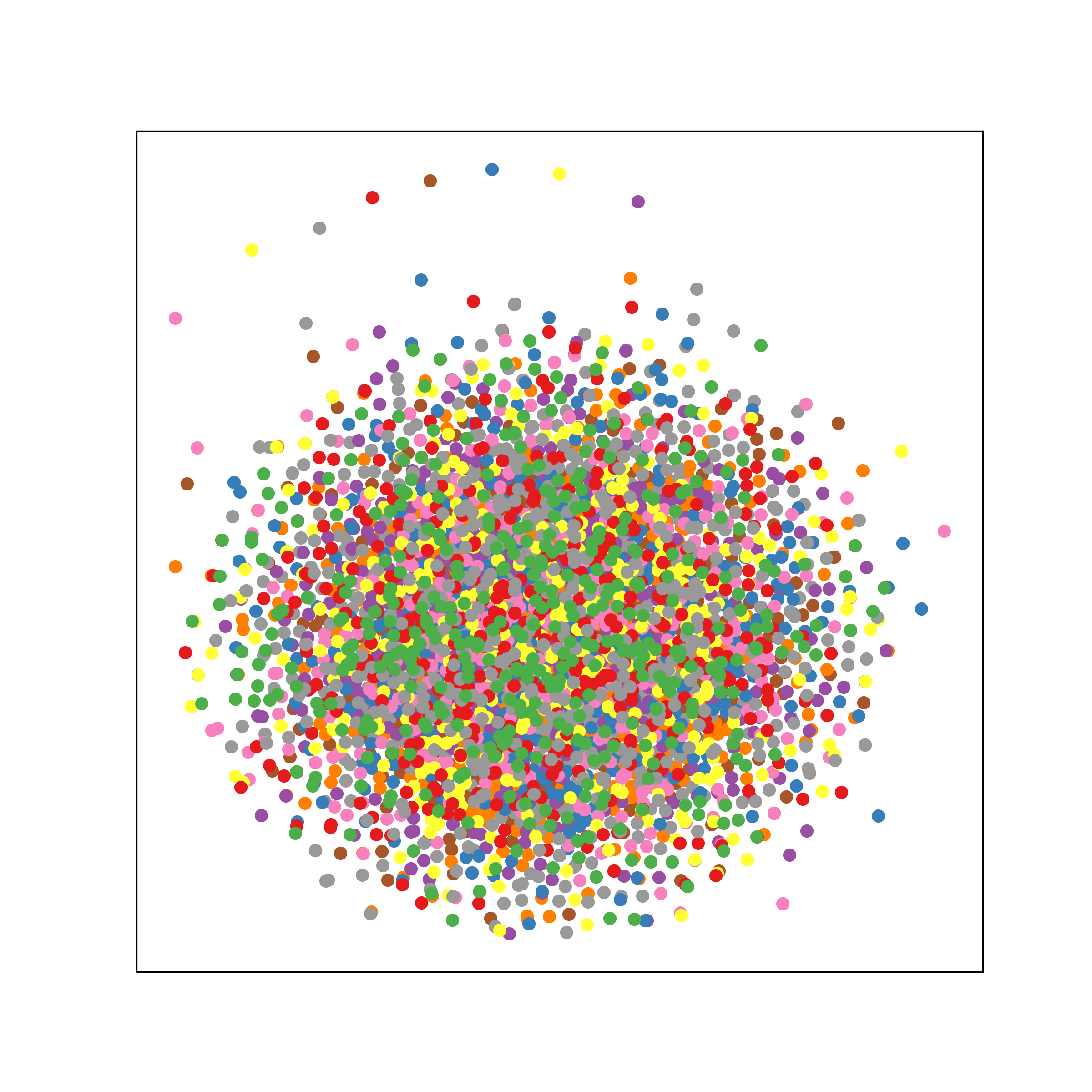}%
    \label{fig:short-a}
}
\hfil
\subfloat[]{
    \includegraphics[width=1.6in]{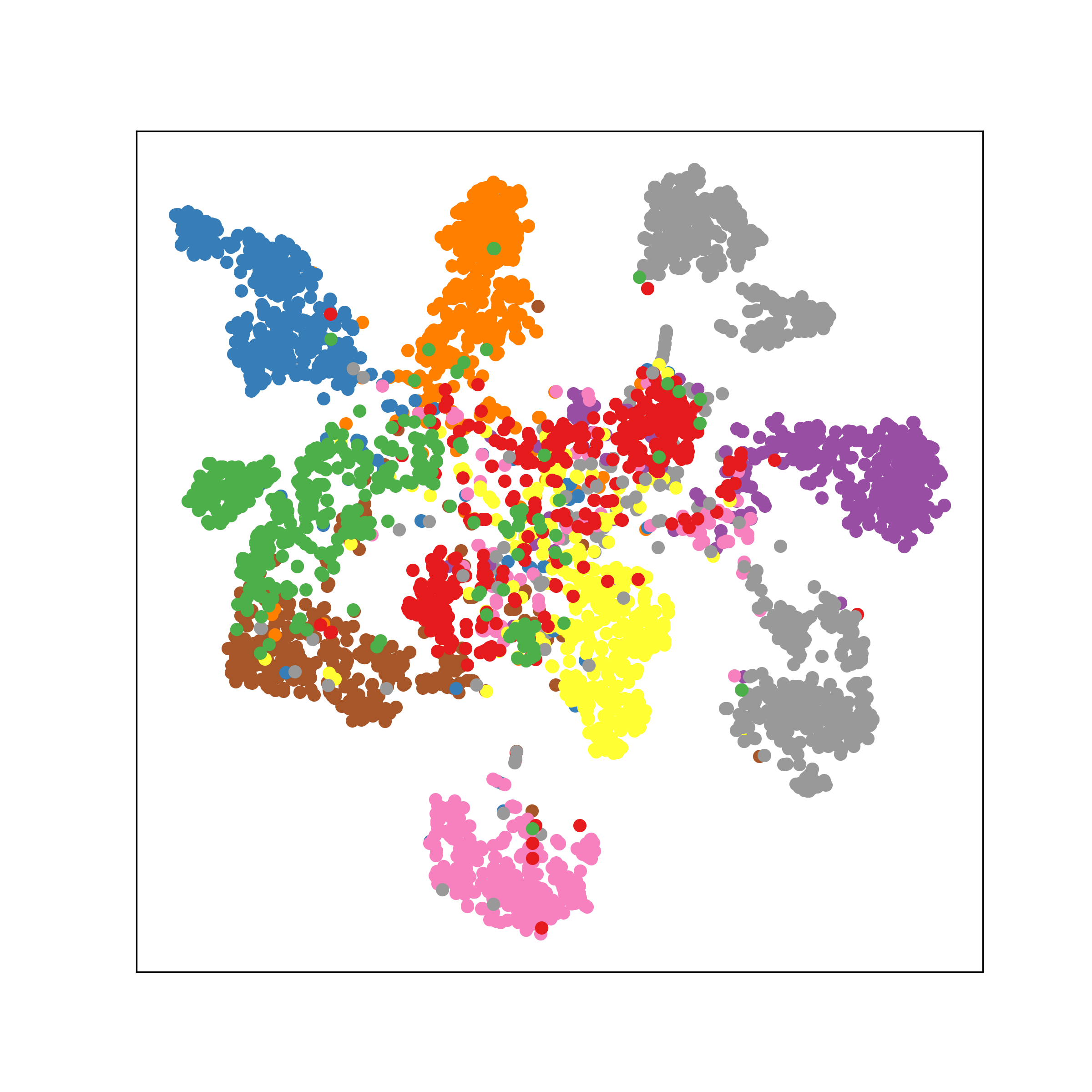}%
    \label{fig:short-b}
}
\caption{The t-SNE~\cite{tsne} visualizes the results of the network in different stages. The left figure is generated at the beginning of the training. The right figure is generated at half of the self-supervised training procedure.}
\label{fig:distribution}
\end{figure}

\subsection{The Role of Cross-Entropy Loss}
For simplicity, we first analyze the situation of supervised learning. Given a batch of images $X = \{x_1,...,x_N\}$ and labels $Y = \{y_1,...,y_N\}$, where batch size is $N$, the corresponding outputs of the image encoder (network) is $Z=\{z_1,...,z_N\}$. All images are classified into $C$ classes. Therefore $z_i$ is a $1$-by-$C$ vector. The loss function is
\begin{equation}
	L = - \frac{1}{|B|}\sum_{i=1}^{|B|}\log{p(y_i|z_i, \tau)},
\end{equation}
where $B$ denotes the batch data and $|B|$ denotes the size of batch. The default value of $\tau$ is 1.
\begin{proposition}
	Assume the network has basic ability to distinguish instances. For images $x_i$ and $x_j$ of the same class $y$ in batch $B$, minimizing loss function is equivalent to maximizing the lower bound of the similarity between two examples' probability distributions.
	\label{prop:similarity}
\end{proposition}
\begin{proof}
	The cross-entropy loss can be expressed by
	\begin{align}
		L =& - \frac{1}{|B|}\sum_{i'=1, i' \notin \{i,j\}}^{|B|}log \; p(y_{i'}|z_{i'}) \nonumber \\
		&- \frac{1}{|B|}(log \; p(y|z_i) + log \; p(y|z_j)),
	\end{align}
	where $x_i$ and $x_j$ belong to the same class $y$. By denoting $\mathbf{P}(z_i) \in \mathbb{R}^C$ as a stack of $p(c|z_i)$ over different classes, which represents the probability distributions for image $x_i$, minimizing the cross-entropy loss may maximize
	\begin{align}
		p(y|z_i) \cdot p(y|z_j) &\leq \sum_{c} p(c|z_i) \cdot p(c|z_j) \nonumber \\
		&\leq S(\mathbf{P}(z_i), \mathbf{P}(z_j)), \nonumber\\
		where \quad S(p, q) &\DEFEQ \frac{\langle p, q \rangle}{\|p\| \cdot \|q\|}.
	\end{align}
	If another instance of the same class can not be found in batch (e.g., can not find $x_j$), the one-hot vector $p(y|z_j) = 1$ and $log \; p(y|z_j) = 0$ can also satisfy our proposition and do not influence the loss. The proof is completed.
\end{proof}
According to Prop.~\ref{prop:similarity}, supervised learning will learn the similarity between two distinct images when minimizing cross-entropy if the network has basic ability to distinguish instances. For those clustering-based methods, cross-entropy loss may also capture this information when the probability distribution of pseudo-labels is sharp. However, as Fig.~\ref{fig:distribution} has shown, it is hard to leverage the similarity between correlated images at the beginning of the training in a self-supervised learning manner. In self-supervised learning such as SwAV and DINO, the cross-entropy loss is likely to draw close uncorrelated images at the beginning of the training (Fig.~\ref{fig:short-a}). After half of the self-supervised training procedure, images of the same classes may be close (Fig.~\ref{fig:short-b}). The cross-entropy loss may learn the similarity between correlated images at this stage. Thus, the problem is how to convert similarity loss into cross-entropy loss naturally during the training.
\begin{figure}[!t]
	\centering
	\includegraphics[width=3in]{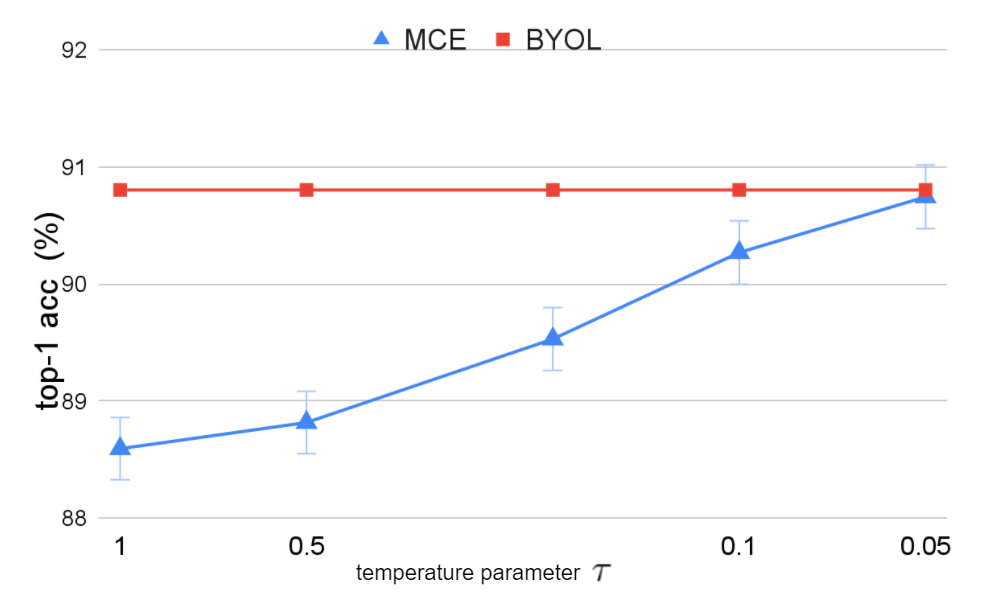}
	\caption{\textbf{Experiments of different $\tau$ in $L_{mce}$.} The top-1 accuracy is obtained by linear evaluation in Imagenette~\cite{imagenette}. The hyper-parameters are consistent with BYOL.}
	\label{fig:temperature_ablation}
\end{figure}
\subsection{Relation between $L_{ce}$ and $L_{similarity}$}
\label{sec:relation section}
Cross-entropy loss is essential for capturing the similarity between two distinct images. However, most clustering-based methods that use $L_{ce}$ may suffer from the poor quality of pseudo-labels. The scale of the dataset may also influence those methods. The hyper-parameters such as the dimensionality of output features may be sensitive. By contrast, contrastive learning based methods may be less affected by this problem. For example, SimSiam works well when the output dimensionality is 2048 in CIFAR-10. However, SwAV works worse when the number of prototypes is 2048 in CIFAR-10. In SimSiam, authors notice that directly replacing $L_{similarity}$ with 
\begin{equation}
	L_{ce}=-\sum_i p(i|z_2, \tau)\log{p(i|q_1, \tau)}
	\label{equ:simsiam_celoss}
\end{equation} may decrease the performance. In SimSiam, they do not use $\tau$. Thus $\tau=1$ here.

To discover the relation between $L_{similarity}$ and $L_{ce}$, we first analyze gradients for $L_{similarity}$ in (\ref{equ:loss_contrastive}). For two vectors $q_1$ and $z_2$ $\in \mathbb{R}^{C}$originated from the same image, the gradients for $q_1$ is
\begin{align}
	\frac{\partial{L}}{\partial{q_1}} = \frac{1}{\|q_1\|}(\frac{\partial{L}}{\partial{\Q1}} - \Q1 \cdot \langle \Q1 , \frac{\partial{L}}{\partial{\Q1}} \rangle),
\end{align}
\begin{align}
	\|\frac{\partial{L}}{\partial{q_1}}\|^2 &= \frac{1}{\|q_1\|^2}(1 - \langle \Q1, \tilde{z}_2 \rangle^2) \leq \frac{1}{\|q_1\|^2},
    \label{equ:q12}
\end{align}
when $L$ is short for $L_{similarity}$. $\|q_1\|^2 \approx C$ at the beginning of the training if we use suitable initializer~\cite{he2015delving}. The gradients for $q_1$ will be bounded.

Then we analyze the cross-entropy loss used in SimSiam. For (\ref{equ:simsiam_celoss}), the gradients for $q_1$ is
\begin{equation}
	\frac{\partial{L_{ce}}}{\partial{q_1}} = \frac{1}{\tau} (\mathbf{P}(q_1) - \mathbf{P}(z_2)), and \; \|\frac{\partial{L_{ce}}}{\partial{q_1}}\|^2 \leq \frac{2}{\tau ^ 2}.
	\label{equ:grad_of_raw_ce}
\end{equation}
Obviously, gradients for $q_1$ in $L_{ce}$ are only influenced by the distance between probability distributions $\mathbf{P}(q_1)$ and $\mathbf{P}(z_2)$. Thus the update may not be under control. Although $L_{ce}$ and $L_{similarity}$ both seem to increase the agreement between two views, these two losses cannot be associated as shown in Fig.~\ref{fig:correlation}.
\begin{figure}[!t]
	\centering
	\includegraphics[width=2.5in]{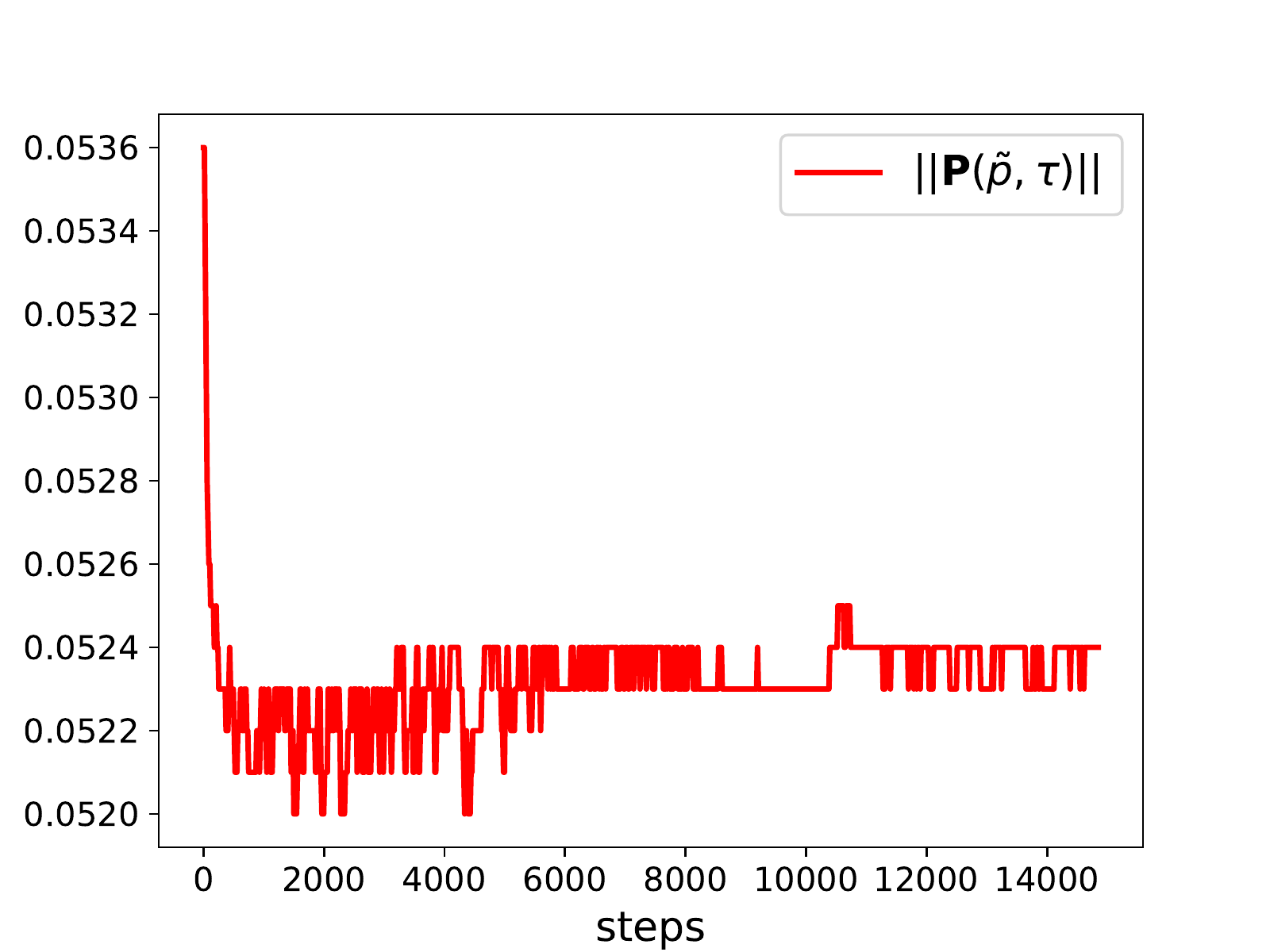}
	\caption{\textbf{$\|\mathbf{P}(z_2)\|$ during the training.} The value of $\|\mathbf{P}(z_2)\|$ is close to the best setting of $\tau$ in Fig.~\ref{fig:temperature_ablation}. In ~(\ref{equ:l2norm}), $\|\mathbf{P}(z_2)\|^2 / \tau^2$ directly influence the magnitude of gradients for $L_{ce}$. The measurement of $\|\mathbf{P}(z_2)\|$ can help choose the hyper-parameters of $\tau$.}
	\label{fig:pz2}
\end{figure}
\begin{figure}[!t]
	\centering
	\includegraphics[width=2.5in]{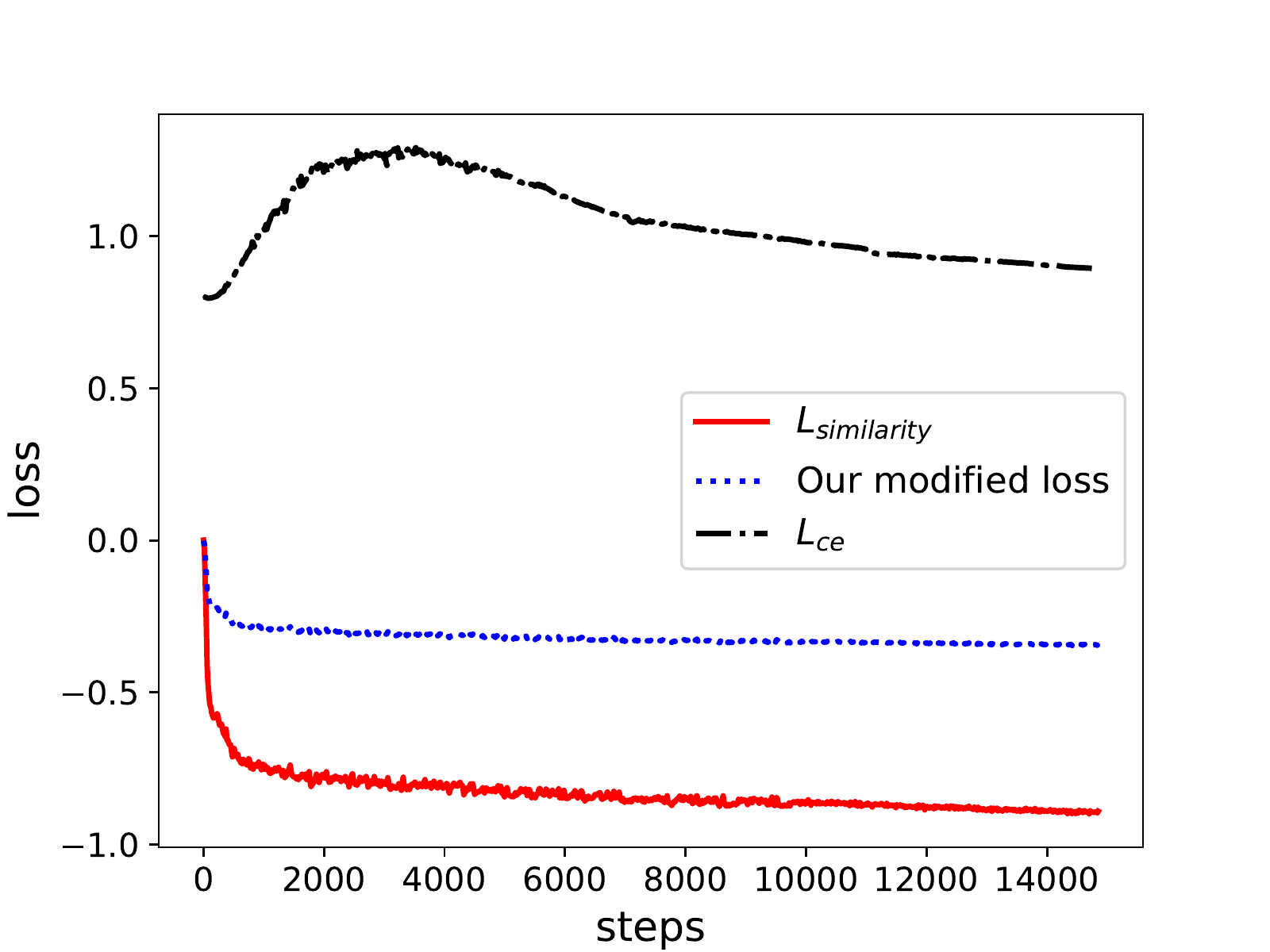}
	\caption{\textbf{The correlation between $L_{similarity}$ and our loss.} The training procedure just minimizes $L_{similarity}$. Other terms (blue curve and black curve) are just recorded to show the correlation with $L_{similarity}$. Our loss is subtracted by a constant $\log{C}$ so that curves can be displayed clearly.}
	\label{fig:correlation}
\end{figure}

Notice that $L_{ce}$ can be expressed by
\begin{equation}
	L_{ce} = (\log{\sum_j \exp{\frac{q_1^{(j)}}{\tau}}}) - (\sum_i p(i|z_2, \tau) \frac{q_1^{(i)}}{\tau}).
	\label{equ:uniformity_and_alignment}
\end{equation}
This equation is similar to the uniformity and alignment term in~\cite{DBLP:conf/icml/0001I20}. The loss is analogous to similarity loss if we ignore the first term. To imitate similarity loss, a modified cross-entropy (MCE) loss can be  expressed by
\begin{equation}
	L_{mce} = - \sum_i p(i|z_2, \tau) \frac{\Q1^{(i)}}{\tau}.
\end{equation}
Instead of using $q_1$ as the input of $softmax$, we use $\Q1$ as the input. This simple modification can lead to
\begin{align}
	\|\frac{\partial{L_{mce}}}{\partial{q_1}}\|^2 &= \frac{1}{\|q_1\|^2}(\|\frac{\partial{L_{mce}}}{\partial{\Q1}}\|^2 - \langle \Q1 , \frac{\partial{L_{mce}}}{\partial{\Q1}} \rangle^2) \nonumber\\
	&= \frac{\|\mathbf{P}(z_2)\|^2}{\tau^2\|q_1\|^2}(1 - \langle \Q1 , \frac{\mathbf{P}(z_2)}{\|\mathbf{P}(z_2)\|} \rangle^2).
	\label{equ:l2norm}
\end{align}
Unlike the gradients for similarity loss, $\|\nabla{q_1} L_{mce}\| \leq \|\mathbf{P}(z_2)\| / \tau$ may lead to smaller gradients if $\|\mathbf{P}(z_2)\|$ is small (e.g., uniform distribution). Therefore, $\tau$ is proved to be crucial for the magnitude of gradients. Fig.~\ref{fig:temperature_ablation} shows the results for different $\tau$. The experimental results convince us that the connection between $L_{similarity}$ and $L_{mce}$ can be established through a suitable $\tau$. According to Fig.~\ref{fig:pz2}, Fig.~\ref{fig:temperature_ablation}, and EQ. (\ref{equ:l2norm}), we can find the range of $\|\mathbf{P}(z_2)\|$ is close to the best settings of $\tau$ in $L_{mce}$. From the perspective of gradients, the increase of $\tau$ for centering mechanism in DINO can be explained. Increasing $\tau$ in $\mathbf{P}(z_2)$ will decrease $\|\mathbf{P}(z_2)\|$, thus the parameters may converge better.

\begin{table*}[!t]
    \caption{The comparison of different loss functions. We provide a detailed comparison of analyzed methods. $L_{align}$ indicates the gradient for the alignment term. For convenience, we provide the gradient magnitude of the alignment term.}
    \label{tab:different_methods}
    \centering
    \begin{tabular}{c|cccc}
        \toprule
        Loss Functions & Alignment Term & Uniformity Term & Gradient Magnitude ($\|\frac{\partial L^{align}}{\partial q_1}\|^2$) & Upper Bound of Gradient Magnitude\\
        \midrule
        $L_{similarity}$ & $-\INNERDOT{\Q1, \Z2}$ & - & $(1 - \langle \Q1, \tilde{z}_2 \rangle^2) / \|q_1\|^2$ & $1 / \|q_1\|^2$\\
        $L_{contrastive}$ & $-\INNERDOT{\Q1, \Z2} / \tau$ & $\log{\sum_{i=0}^{|B|} \exp{\INNERDOT{q, k_i} / \tau}}$ & $(1 - \langle \Q1, \tilde{z}_2 \rangle^2) / \tau^2\|q_1\|^2$ & $1 / \tau^2\|q_1\|^2$\\
        $L_{ce}$ & $-\INNERDOT{q_1, \mathbf{P}(z_2)} / \tau$ & $\log{\sum_j \exp{q_1^{(j)} / \tau}}$ & $\|\mathbf{P}(z_2)\|^2 / \tau^2$ & $1 /\tau^2$ \\
        $L_{mce}$ & $-\INNERDOT{\Q1, \mathbf{P}(z_2)} / \tau$ & - & $(1 - \langle \Q1 , \frac{\mathbf{P}(z_2)}{\|\mathbf{P}(z_2)\|} \rangle^2) \cdot \|\mathbf{P}(z_2)\|^2 / \tau^2\|q_1\|^2$ & $\|\mathbf{P}(z_2)\|^2 / \tau^2\|q_1\|^2$ \\
        Ours ($L_{iccl}$) & $-\INNERDOT{\Q1, \mathbf{P}(z_2)} / \tau_1$ & $\sum_{i} q(i) log \frac{q(i)}{p(i|\theta)}$ & $(1 - \langle \Q1 , \frac{\mathbf{P}(z_2)}{\|\mathbf{P}(z_2)\|} \rangle^2) \cdot \|\mathbf{P}(z_2)\|^2 / \tau_1^2\|q_1\|^2$ & $\|\mathbf{P}(z_2)\|^2 / \tau_1^2\|q_1\|^2 \approx 1 / \|q_1\|^2$ \\
    \end{tabular}
\end{table*}

\subsection{Detailed Method}
$L_{ce}$ may capture class-level information and $L_{similarity}$ may capture instance-level information. Based on the aforementioned analyses, we propose a simple method to leverage the relation between $L_{ce}$ and $L_{similarity}$.

There are two $\tau$ in (\ref{equ:simsiam_celoss}). We denote $\tau$ for $q_1$ as $\tau_1$ and $\tau$ for $z_2$ as $\tau_2$. $\tau_1$ and $\tau_2$ have completely distinct roles on gradients. In essence, $\tau$ in (\ref{equ:l2norm}) is $\tau_1$, which directly influences the magnitude of $\|\frac{\partial{L_{mce}}}{\partial{q_1}}\|^2$. $\tau_2$ adjusts the magnitude of $\|\mathbf{P(z_2)}\|$. $\tau_1$ and $\Q1$ may affect the gradients of $q_1$, which is essential to construct the relation between $L_{ce}$ and $L_{similarity}$. The above analysis explains why we can set different values for $\tau_1$ and $\tau_2$. As (\ref{equ:l2norm}) shows, 
a basic setting of $\tau_1$ to be adaptive is $\tau_1 = \|\mathbf{P}(z_2)\|$. We also provide some detailed analysis for the setting of adaptive $\tau_1$ in the appendix. Moreover, the difference between (\ref{equ:l2norm}) and (\ref{equ:grad_of_raw_ce}) indicates that the $l_2$-norm of $p$ is essential for getting suitable gradients. Therefore, the loss function is
\begin{align}
	L_{iccl}&=-\sum_i p(i|z_2, \tau_2)\log{p(i|\frac{q_1}{\|q_1\|}, \tau_{1})} \nonumber \\
	&=(\log{\sum_j \exp{\frac{\Q1^{(j)}}{\tau_{1}}}}) - (\sum_i p(i|z_2, \tau_2) \frac{\Q1^{(i)}}{\tau_{1}}).
	\label{equ:iccloss}
\end{align}

This loss function is similar to the loss in DINO and SwAV. However, the inputs of the loss in DINO and SwAV are not $l_2$-normalized. Moreover, $\tau_1$ for $\Q1$ is used to control the magnitude of $\|\nabla{q_1} L_{mce}\|$ in this formula. In SwAV and DINO, $\tau_1$ may be used to generate a basic probability distribution. As Fig.~\ref{fig:correlation} shows, our loss function can be associated with $L_{similarity}$. This property may prevent our loss from unbalanced clustering and provide more reasonable training. For instance, we can use $L_{similarity}$ at the beginning of the training and use $L_{iccl}$ when the network can basically extract instance-level features.

The relation between $L_{ce}$ and $L_{similarity}$ helps networks benefit from both instance-level information and class-level information. The derived method is less affected by hyper-parameters and balancing mechanisms (e.g., Sinkhorn-Knopp algorithm), which is different from certain clustering-based methods. Our method can be simply understood as focusing on the dimensionality of larger values to provide cross-instance supervision.

\subsection{The Uniformity for probabilities}
In SwAV and DINO, the supervised labels $\mathbf{P}(z_2)$ are generated through balancing mechanisms, such as Sinkhorn-Knopp algorithm and moving average centering, to prevent unbalanced clustering. Traditional cross-entropy loss may not prevent trivial solution. Based on the relation between $L_{similarity}$ and $L_{iccl}$, our method can be less suffered from unbalanced clustering. In our method, we just add some regularization for the loss function. We assume that the outputs of the batch data $B$ is $P=\{q_1,...,p_N\}$. The uniformity assumption indicates that the outputs in each dimension should be approximately close, which can be expressed by
\begin{align}
	p(i|\theta) = \frac{1}{|B|} \sum_{j=1}^{|B|} p(i|p_j, \tau, \theta), \nonumber\\
	\mathop{\min}_{\theta} D_{KL}(P||Q) = \sum_{i} q(i) log \frac{q(i)}{p(i|\theta)},
	\label{equ:uniformity}
\end{align}
where $q(i) = \frac{1}{C}$ represents uniform distribution. $\theta$ represents the parameters. $P$ and $Q$ are denoted as the estimated probability distribution and expected probability distribution, respectively. We use KL-divergence between $P$ and $Q$ in (\ref{equ:uniformity}). The bias of representations will be used to update the network parameters. $\lambda_r$ is used to adjust the strength of the uniformity regularization. The final loss is
\begin{equation}
	L_{final} = L_{iccl} + \lambda_r \times (\sum_{i} q(i) log \frac{q(i)}{p(i|\theta)}).
\end{equation}
By default, half of the training use $L_{similarity}$ to maintain instance-level information, and half of the training use $L_{final}$ to maintain class-level information.

\subsection{The Details among Different Losses}
Table~\ref{tab:different_methods} shows the relation among different loss functions. For convenience, the gradient magnitude of the alignment term is provided. The alignment term indicates how this method learns the information between two correlated representations. We also provide the upper bound of the corresponding gradient magnitude for each loss function. Table~\ref{tab:different_methods} clarifies the relation of our method with other loss functions. $L_{similarity}$ has a smaller upper bound than $L_{ce}$ due to $\tau \ll \|q_1\|$ (which has been analyzed in Sec~\ref{sec:relation section}). $L_{similarity}$ is similar to our method in the upper bound of gradient magnitude, which is the core difference between our method and $L_{ce}$. This point makes $L_{similarity}$ be replaceable with our method. Moreover, the alignment term of our method and $L_{ce}$ are similar, indicating that our method may leverage the cross-entropy loss to learn the similarity between intra-class instances. Therefore, our method can use $L_{similarity}$ at the beginning of the training and replace $L_{similarity}$ with $L_{iccl}$ to learn intra-class information after several epochs.

\section{Experiments}
In this section, we conduct a series of experiments on model designs for self-supervised representation learning.

\subsection{Baseline Settings}
Our method can be easily combined with BYOL and SimSiam. We follow the BYOL settings as our baseline. Specifically, the default temperature $\tau_2$ equals 0.07 for all datasets. $\tau_1$ is 0.1 as the default. We use a cosine decay learning rate schedule~\cite{loshchilov2016sgdr} for all experiments. All augmentation strategies and initialization methods are the same as BYOL. For ResNet~\cite{he2016identity}, we initialize the scale parameters as 0~\cite{goyal2017accurate} in the last Batch Normalization (BN)~\cite{DBLP:conf/icml/IoffeS15} layer for every residual block. All our models are trained by mixed-precision to accelerate training speed. The augmentation strategies and initialization for each method are consistent to make a fair comparison. The detail of augmentation and initialization can be found in the appendix.

\subsubsection{Imagenette settings}
We use Imagenette~\cite{imagenette} to conduct basic experiments. Following BYOL's settings, we use LARS~\cite{you2017large} with base learning rate ($lr$) = 2.0 for 1000 epochs, weight decay = 1e-6, momentum = 0.9, and batch size = 256.
According to ~\cite{SimSiam}, the $lr$ is $(\text{base}\;lr) \times \frac{BatchSize}{256}$. The backbone is ResNet-18. The projector is a 3-layer multi-layer perceptron (MLP), and the predictor is a 2-layer MLP. The output dimensionality is 512. We do not use momentum encoder here.

\subsubsection{ImageNet settings}
We use ImageNet~\cite{krizhevsky2012imagenet} to verify our representations. We use LARS with base $lr$ = 0.3, weight decay = 1e-6, momentum = 0.9, $\lambda_r = 5$, and batch size = 1024. The backbone is ResNet-50. The projector is a 3-layer MLP with output dimensionality 256. The predictor is a 2-layer MLP with output dimensionality 256. The momentum for momentum encoder is 0.99.

\subsubsection{Linear evaluation}
Given the pre-trained network, we train a supervised linear classifier on frozen features (after average pooling from ResNet). For Imagenette, the classifier uses base $lr$ = 0.2, weight decay = 0, momentum = 0.9, epoch = 100, and batch size = 4096. The optimizer is SGD with Nesterov. For ImageNet, we follow settings in ~\cite{barlowtwins}. The linear classifier training uses base lr = 0.3 with a cosine decay schedule for 100 epochs, weight decay = 1e-6, momentum = 0.9, batch size = 256 with SGD optimizer.

\begin{figure}[!t]
	\centering
	\includegraphics[width=3in]{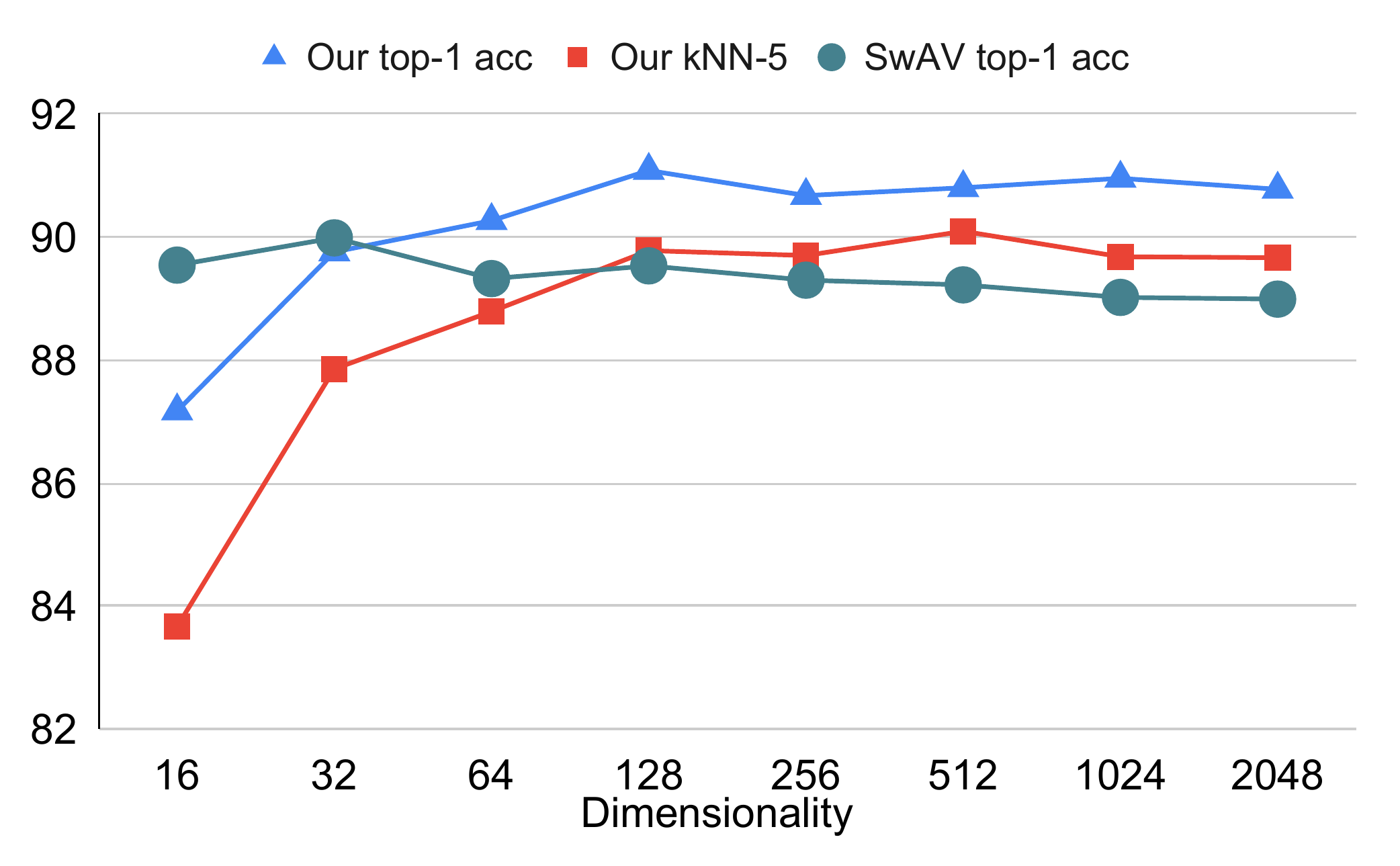}
	\caption{\textbf{Results versus output dimensionality.} kNN-5 denotes the result of $K$-Nearest Neighbor when $K$ is 5.}
	\label{fig:dimension}
\end{figure}

\subsection{Analysis of Hyper-Parameters}
\subsubsection{Hyper-parameter $\lambda_r$}
In our method, we use $\lambda_r$ to regularize the uniformity of $q_1$. Tab.~\ref{tab:tau_and_ratio} shows the results for different $\lambda_r$. Our method may be less affected by incorrect clustering due to the correlation with $L_{similarity}$. The instance-level learning reduces the dependence on gradients generated by clustering. Compared with SwAV and DINO, there is no need for our approach to impose any balancing mechanism. Our approach works well when $\lambda_r = 0$. By contrast, SwAV relies on Sinkhorn-Knopp algorithm and DINO relies on centering.

\begin{table}[!t]
	\centering
    \caption{\textbf{The analyses of hyper-parameters $\lambda_r$ and $\tau_2$.} The results are top-1 accuracy (\%) in Imagenette, which is the average over three independent experiments.}
	\label{tab:tau_and_ratio}
		\begin{tabular}{c|ccccccc}
			\toprule
			& \multicolumn{7}{c}{$\lambda_r$} \\
			$\tau_2$ & $0$ &$0.5$ &$1.0$ &$2.5$ &$5.0$ &$7.5$ &$10.0$\\
			\midrule
			0.05&90.77 & 90.90 & 91.04 & 90.91 & 90.95 & 90.72 & 90.76\\
			0.07&91.18 & 91.15 & 91.21 & 91.14 & 91.08 & 90.71 & 90.71\\
			0.1&90.74 & 90.93& 91.02& 90.89& 91.05& 90.79& 90.56\\
		\end{tabular}
\end{table}

\subsubsection{Hyper-parameter $\tau_2$}
Fig.~\ref{fig:temperature_ablation} shows the results for different $\tau$. As the $\tau$ becomes smaller, the performance becomes better. This phenomenon is consistent with $(\ref{equ:l2norm})$. An inappropriate $\tau$ will magnify or diminish the magnitude of $\|\nabla{q_1} L\|$. Tab.~\ref{tab:tau_and_ratio} analyzes the influence of $\tau_2$. We find that $\tau_2=0.07$ can provide a better result in this dataset. In fact, $\tau_2$ only has an impact on $\|\mathbf{P}(z_2)\|$. $\|\mathbf{P}(z_2)\|$ will only be influenced by the dimensionality of features and $\tau_2$. Therefore, $\tau_2=0.07$ may be extended to other datasets.

\subsubsection{Hyper-parameter of output dimensionality}
Fig.~\ref{fig:dimension} shows the results of different dimensionalities. Our method is stable. Although the dataset only has 10 categories. Our method has a good result when the dimensionality is large. By contrast, we find certain clustering-based methods cannot get a good result when the number of prototypes is excessive in Imagenette. This point indicates that the relation between $L_{similarity}$ and $L_{iccl}$ helps decrease the influence of incorrect clustering.

\subsubsection{Hyper-parameter of batch size}
Fig.~\ref{fig:batchsize} shows the results of different batch sizes. We follow BYOL to accumulate the gradients of $N$ steps so that the learning rate is not changed. As we conjecture, the performance of our method is less influenced by batch size, which is similar to contrastive learning based methods.

\begin{figure}[!t]
	\centering
	\includegraphics[width=0.9\linewidth]{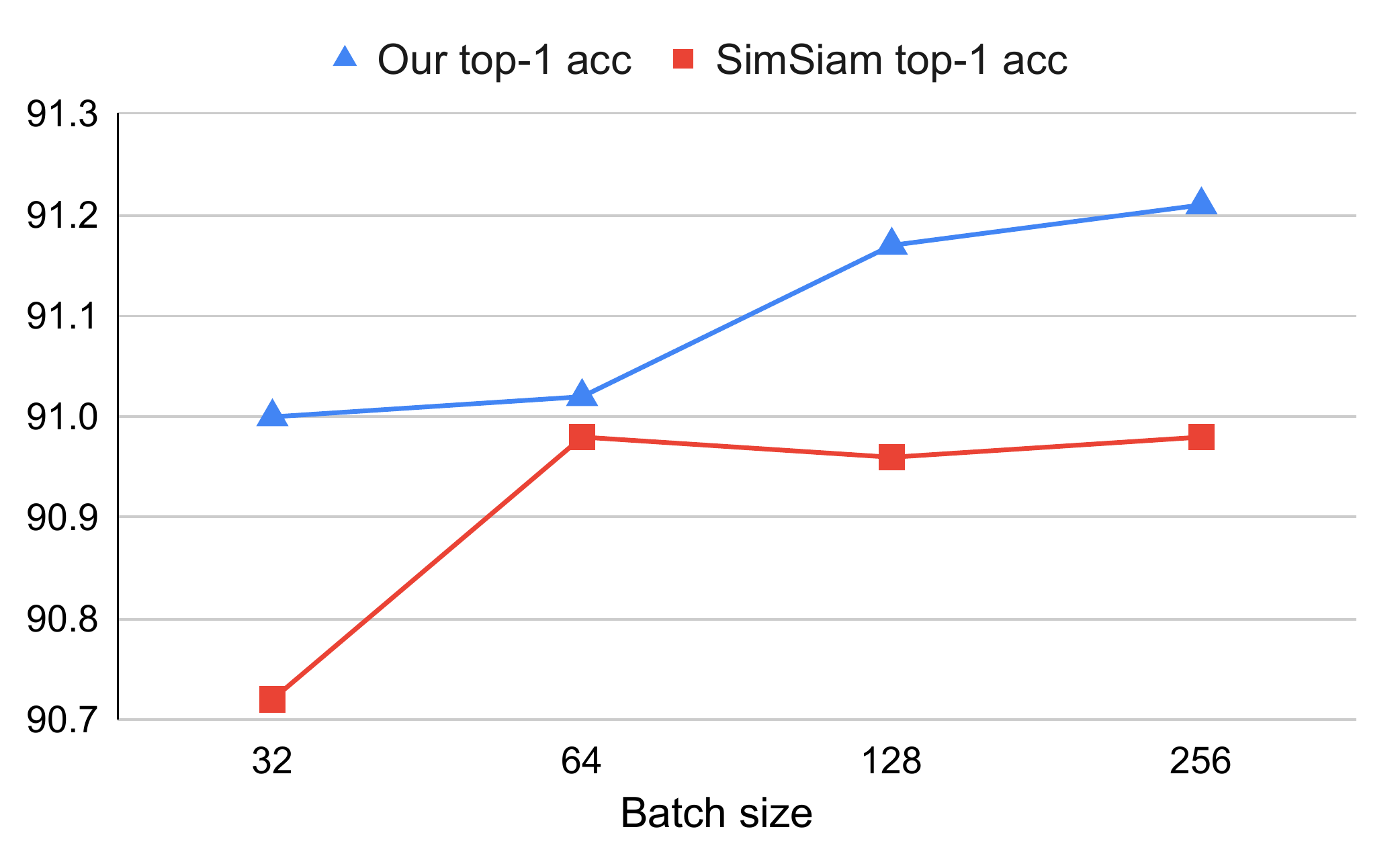}
	\caption{\textbf{Experiments of different batch sizes.}}
	\label{fig:batchsize}
\end{figure}

\begin{table}[!t]
	\centering
    \caption{\textbf{The ablation study of $f(i, z_2)$ in (\ref{equ:traditional_ce}).} SK is short for Sinkhorn-Knopp algorithm. Centering indicates the centering procedure in DINO. Here we just change $f(i,z_2)$. Thus we use $\tilde{p}$ rather than $p$ as the input of $Softmax$.}
	\label{tab:ablation_f}
	\resizebox{.9\columnwidth}{!}{
		\begin{tabular}{cccc}
			\toprule
			SK & Centering& $p(i|z_2, \tau = \tau_1)$ & $p(i|z_2, \tau = \tau_2 < \tau_1)$\\
			\midrule
			90.85 & 90.49 &90.74 & 91.21\\
		\end{tabular}
	}
\end{table}

\begin{table*}[!t]
	\centering
	\caption{\textbf{The results of different self-supervised methods in ImageNet.} All results are pretrained with two 224$\times$224 views. All methods use ResNet-50 as backbone. We do not use any tricks (e.g., multi-crop in SwAV and fixing lr in SimSiam). * represents the result of 400 epochs. $\ddagger$ denotes the result with Multi-Crop (this trick may boost the result). Top-4 best self-supervised methods are \underline{underlined}.}
	\label{tab:imagenet_result}
	\resizebox{.8\linewidth}{!}{
		\begin{tabular}{cccc|cc}
			\toprule
			&&&&\multicolumn{2}{c}{Top-1 Acc (\%)}\\
			Method & Basic Loss & Batch Size & Dimensionality& 100 epochs & 300 epochs \\
			\midrule
			\multicolumn{5}{l}{\textbf{Contrastive Methods}} \\
			\midrule
			SimCLR~\cite{SimCLR} & InfoNCE & 4096 & 2048 & 66.5 & *69.8\\
			MoCo v2~\cite{mocov2} & InfoNCE & 256 (65536 queue size) & 256 & \underline{67.4} & *71.1\\
			MoCo v3~\cite{MoCov3} & InfoNCE & 1024 & 256 & \underline{68.1} & \underline{\textbf{72.3}}\\
			BYOL~\cite{BYOL} & $L_{similarity}$ & 1024 & 256 & 66.0 & \underline{72.2}\\
			SimSiam~\cite{SimSiam} & $L_{similarity}$ & 256 & 2048 & 67.3 & *70.8\\
			& $L_{ce}$ & 256 & 2048 & 63.2 & - \\
			\midrule
			BarlowTwins~\cite{barlowtwins} & BarlowTwins Loss &  1024 & 8192 & 67.4 & 71.4\\
			\midrule
			\multicolumn{5}{l}{\textbf{Clustering-based Methods}} \\
			\midrule
			SeLa~\cite{DBLP:conf/iclr/AsanoRV20a} & $L_{ce}$ & 4096 & 3000 & 61.5 & *67.2\\ 
			SwAV~\cite{SwAV} & $L_{ce}$ & 256 (4096 queue size) & 3000 & 66.5 & *70.7\\
			DINO~\cite{dino} & $L_{ce}$ & 1024 & 65536 & \underline{67.8} & \underline{$\ddagger$72.1}\\
			\midrule
			Ours & $L_{ce}$ & 1024 & 256 & \underline{\textbf{68.2}} & \underline{71.7} \\
			Ours & $L_{ce}$ & 1024 & 512 & \underline{68.1} & \underline{71.5} 
		\end{tabular}
	}
\end{table*}

\begin{table}[!t]
	\centering
    \caption{\textbf{The results of different methods in Imagenette.} $\dag$ indicates the methods where default settings cannot be transferred to Imagenette experiments, and we report the best result. $\ddag$ indicates the method without uniformity regularization (Sinkhorn-Knopp in SwAV, centering in DINO, and $\lambda_r$ in our method).}
	\label{tab:imagenette_result}
	\resizebox{.9\columnwidth}{!}{
		\begin{tabular}{c|c|c}
			\toprule
			Method & Top-1 Acc (\%) & KNN-5 Top-1 Acc (\%) \\
			\midrule
			\multicolumn{3}{l}{\textbf{Contrastive Methods}} \\
			\midrule
			SimSiam & \textbf{90.98}  & 89.90 \\
			MoCo & 90.39  & 88.81 \\
			BYOL & {90.87} & \textbf{90.26 }\\
			\midrule
			BarlowTwins$\dag$ & 87.06 & 87.31 \\
			\midrule
			\multicolumn{3}{l}{\textbf{Clustering-based Methods}} \\
			\midrule
			SwAV$\dag$ & 89.99 & 88.61 \\
			DINO$\dag$ & 88.86 & 88.41 \\
			$\ddag$SwAV & fail & -\\
			$\ddag$DINO & fail & -\\
			\midrule
			$\ddag$Ours & 91.18  & 89.98\\
			Ours (fixed $\tau_1$) & 91.21  & \textbf{90.31 }\\
			Ours (adaptive $\tau_1$) & \textbf{91.23 } & 89.94 \\
		\end{tabular}
	}
\end{table}

\subsection{Ablations on Loss Function}
First, as Fig.~\ref{fig:correlation} shows, $\|p\|$ is essential to build the relation with $L_{similarity}$. The main difference between our loss and $L_{ce}$ is the input of $Softmax$. Based on the analyses of gradients, $\|p\|$ is a more suitable form to feed into $Softmax$. This point of view is pivotal to our method.

We choose different methods to generate probability distribution to supervise the update of $q_1$. Tab.~\ref{tab:ablation_f} shows results. Sinkhorn-Knopp algorithm and centering adjust the probability distribution based on the batch of data. The magnitude of $\|\mathbf{P}(z_2)\|$ may be small, although a smaller $\tau_2$ is used. However, our method does not have a centering mechanism, which may lead to a large $\|\mathbf{P}(z_2)\|$. Moreover, this readjustment may disturb the pseudo-labels and confuse the training. Then in the case when $\tau_2 = \tau_1$, the supervised probability may not be sharp, which loses the ability to capture the similarity between distinct images.

\subsection{Comparison of Other Methods}
We first compare our method with other self-supervised methods in Imagenette. Tab.~\ref{tab:imagenette_result} shows the results for different self-supervised methods. We find that SimSiam, MoCo, and BYOL can be easily extended to this dataset, indicating that those methods are robust to different datasets. BarlowTwins focuses on the correlation of different channels. We find this method is similar to those clustering-based methods. Large dimensionality is not suitable for this dataset. SwAV and DINO are clustering-based methods. The hyper-parameters are set for ImageNet. We find those hyper-parameters are less useful in this dataset. We change the number of prototypes and choose the best result. However, the performances are still worse than results of contrastive learning based methods. SwAV and DINO both heavily rely on uniformity regularization. In ImageNet and Imagenette, those methods may fail without uniformity regularization. Our method leverages the instance-level information and class-level information through (\ref{equ:iccloss}). Therefore, our method may be less suffered from the problem of incorrect clustering. Furthermore, $\tau_1$ and $\tau_2$ are set manually in SwAV and DINO. In our method, adaptive $\tau_1$ can produce a competitive result. This point confirms (\ref{equ:l2norm}) and the analyses of gradients.

Based on the hyper-parameter in Imagenette, we conduct the experiments in ImageNet. We find these hyper-parameters can still work well in ImageNet. Tab.~\ref{tab:imagenet_result} shows the results in ImageNet. We analyze results from the perspectives of loss function and dimensionality. All backbones are ResNet-50. The setting of $\tau_1$ is fixed as 0.1. This may be the problem of hyper-parameters, and we will find suitable hyper-parameters in the future.

\subsubsection{Loss function}
Traditional contrastive learning based methods use similarity loss. For SimSiam, authors find replacing $L_{similarity}$ with $L_{ce}$ may lead to an inferior result (67.3 and 63.2). The results of those contrastive learning based methods may be less influenced by dimensionality. However, methods that use InfoNCE~\cite{oord2018representation} may rely on large batch size. MoCo requires a large queue size or large batch size to maintain a good result. BYOL is stable in both ImageNet and Imagenette. In general, similarity loss may be a good manner to capture instance-level information. BarlowTwins attempts to leverage the correlation of different dimensionalities. This method may benefit from large dimensionality. When the output dimensionality is 8192, Barlowtwins has a competitive result. However, the performance may decrease a lot when the output dimensionality becomes smaller~\cite{barlowtwins}. Thus, the hyper-parameter may be suitable for ImageNet but not suitable for other datasets. SwAV and DINO both leverage $L_{ce}$ to do online clustering. As we propose in Prop.~\ref{prop:similarity}, $L_{ce}$ may help to correlate different instances. Our method establishes the relation between $L_{similarity}$ and $L_{ce}$. Therefore, our method also captures class-level information. From this perspective, our method may leverage more information than other methods. When class-level information is hard to capture, the method may use instance-level information to provide robust training.
\subsubsection{Dimensionality}
In experiments, we find those clustering-based methods may be sensitive to the scale of the dataset. The hyper-parameter of the number of prototypes for those methods may not act well in tiny datasets. In ImageNet, those methods still need a large dimensionality. The decrease in the number of prototypes or the data diversity may affect those methods. However, as Fig.~\ref{fig:dimension} and Tab.~\ref{tab:imagenet_result} show, our method is similar to those contrastive learning based methods. When the category is 10, our method can get a competitive result with dimensionality 2048. When the category is 1000, our method can also get a competitive result with dimensionality 256. Moreover, the use of $L_{ce}$ helps to discover the similarity between distinct images. On the contrary, SimSiam gets an inferior result with $L_{ce}$. The established relation between $L_{ce}$ and $L_{similarity}$ boosts the learned information.

\begin{table*}[!t]
	\centering
    \caption{\textbf{Transfer Learning.} VOC 07 det: Faster R-CNN~\cite{ren2015faster} fine-tuned in VOC 2007 trainval, evaluated in VOC 2007 test. VOC 07+12 det: Faster R-CNN fine-tuned in VOC 2007 trainval + 2012 trainval, evaluated in VOC 2007 test. Methods that use ViT~\cite{dosovitskiy2020image} as backbone may not be compatible with Faster R-CNN. We also provide the transfer learning in COCO~\cite{coco} in appendix.}
	\label{tab:object_detection}
	\resizebox{0.5\linewidth}{!}{
		\begin{tabular}{c|cccccc}
			\toprule
			Method & \multicolumn{3}{c}{VOC 07 det} & \multicolumn{3}{c}{VOC 07+12 det} \\
			\cline{2-4} \cline{5-7}
			& $AP_{all}$ & $AP_{50}$ & $AP_{75}$&$AP_{all}$ & $AP_{50}$ & $AP_{75}$ \\
			\midrule
			Supervised &42.4 & 74.4 & 42.7 & 53.5 & 81.3 & 58.5 \\
			\midrule
			SimCLR & 46.8 & 75.9 & 50.1 & 55.5 & 81.8 & 61.4 \\
			MoCo v2 & \textbf{48.5}& 77.1 & \textbf{52.5} & \textbf{57.0} & 82.5 & 63.3 \\
			BYOL & 47.0 & 77.1 & 49.9& 55.3 & 81.4 & 61.1 \\
			SwAV &46.5 & 75.5& 49.6 & 55.4 & 81.5 & 61.4 \\
			SimSiam (optimal) & \textbf{48.5} & \textbf{77.3} & \textbf{52.5}& \textbf{57.0} & 82.4 & \textbf{63.7} \\
			BarlowTwins & - & - & - & {56.8} & \textbf{82.6} & 63.4 \\
			\midrule
			Ours & 47.2& 75.7 & 51.4& 55.5& 81.9 & 61.5\\
		\end{tabular}
	}
\end{table*}
\begin{table*}[!t]
	\centering
    \caption{\textbf{Transfer Learning.} COCO detection and COCO instance segmentation: Mask R-CNN C-4~\cite{maskrcnn} (2x schedule) fine-tuned in COCO 2017 train~\cite{coco}, evaluated in COCO 2017 val.}
	\label{tab:object_detection_coco}
	\resizebox{0.5\linewidth}{!}{
		\begin{tabular}{c|ccc|ccc}
			\toprule
			Method & \multicolumn{3}{c}{COCO detection} & \multicolumn{3}{c}{COCO instance seg.} \\
			\cline{2-4} \cline{5-7}
			& $AP$ & $AP_{50}$ & $AP_{75}$&$AP^{mask}$ & $AP^{mask}_{50}$ & $AP^{mask}_{75}$ \\
			\midrule
			\multicolumn{7}{l}{\textbf{1x schedule}} \\
			\midrule
			Supervised & 38.2& 58.2& 41.2 & 33.3& 54.7& 35.2\\
			\midrule
			SimCLR & 37.9 & 57.7 & 40.9 & 33.3 & 54.6 & 35.3 \\
			MoCo& \textbf{39.2}&58.8&42.5&34.3&55.5&36.6\\
			BYOL & 37.9&57.8&40.9&33.2&54.3&35.0\\
			SwAV & 37.6&57.6&40.3&33.1&54.2&35.1\\
			SimSiam& \textbf{39.2}&\textbf{59.3}&42.1&\textbf{34.4}&\textbf{56.0}&\textbf{36.7}\\
			BarlowTwins & \textbf{39.2}&59.0&\textbf{42.5}&34.3&\textbf{56.0}&36.5\\
			Ours & 38.4& 58.3& 41.2& 33.2& 54.7& 35.0\\
			\midrule
			\multicolumn{7}{l}{\textbf{2x schedule}} \\
			\midrule
			Supervised & 40.0& 59.9& 43.1 & 34.7& 56.5& 36.9\\
			\midrule
			MoCo& 40.7&60.5&44.1&35.4&57.3&37.6\\
			Ours & 39.9& 59.8& 43.2& 34.9& 56.6& 37.1\\
		\end{tabular}
	}
\end{table*}

\subsection{Transfer to other tasks}
Following SimSiam~\cite{SimSiam,TransferLearning}, we conduct several transfer learning experiments. In Tab.~\ref{tab:object_detection}, we compare the representation quality by transfer learning. We fine-tune the parameters in the VOC~\cite{VOC} datasets. The experimental settings follow the codebase from ~\cite{wu2019detectron2}. We find our pretrained model does not have a competitive result with certain self-supervised methods. We conjecture that our method attempts to learn class-level agreement~\cite{ericsson2021well}. In fact, our intention is to learn class information (the category of an image) during the self-supervised training procedure. This attempt may weaken the performance of object detection. Mid-level information may be discarded. However, as mentioned in other self-supervised methods, self-supervised training may provide a superior result to supervised learning. 

Table~\ref{tab:object_detection_coco} shows transfer learning results on COCO dataset. For COCO dataset, we only find the codebase in MoCo's GitHub. Therefore, we follow the settings in MoCo. The performance in COCO dataset is consistent with the performance in VOC dataset. The performance of our method is close to the best method. In fact, we find the learning rate for MoCo on COCO and VOC may not be suitable for our pre-trained model. We may search for an appropriate hyper-parameter for transfer learning in our later version.

\section{Conclusions}
Our method is conceptually analogous to SwAV and DINO. All these methods leverage feature-level cross-entropy to do unsupervised learning. However, SwAV and DINO need approaches to balance the probability distribution. For example, SwAV uses Sinkhorn-Knopp algorithm to balance the probability distribution of all instances in the batch. DINO uses centering on accumulating the bias of probability distribution. The centering mechanism may modify the intensity of different prototypes in the subsequent training. In SwAV and DINO, authors emphasize the importance of maintaining uniformity. However, in this paper, we reduce the dependence on uniformity mechanisms. The perspective of gradients leads us to find a loss function that may have similar behavior as the similarity loss. This point is critical to get rid of uniformity mechanisms.

Our method uses a completely different loss function from those contrastive learning based methods. However, the approach is correlated with similarity loss through the derivation of gradients. This perspective helps our approach to maintain the robustness of those contrastive learning based methods. A reasonable gradient also provides a stable and smooth training perspective compared with those methods which directly feed $q_1$ into $Softmax$. As shown in Fig.~\ref{fig:correlation}, the input of $Softmax$ is crucial to establish the correlation. In fact, our method can be interpreted as conducting similarity loss through the probability distribution. This perspective explains why our method is less affected by dimensionality and uniformity regularization. By maximizing the probability distribution of instances, the method may implicitly learn the similarity between images.

\section*{Acknowledgments}
This work was supported by the National Key R\&D Program of China (Grant No. 2022YFF0711404), Natural Science Foundation of Jiangsu Province (Grant No. BK20201250), the National Science Foundation of China under Grant 62172090, CAAI-Huawei MindSpore Open Fund, Nanjing Municipal Program for Technological Innovation by Overseas Scholars Under Grant No. 1109012301 and Start-up Research Fund of Southeast University under Grant RF1028623097. We thank the Big Data Computing Center of Southeast University for providing the facility support on the numerical calculations in this paper. Jie Gui is the corresponding author of this paper.

{\appendix [Derivation of Equations in Main Paper]

\subsection*{Derivation of EQ. (9) and EQ. (10) in Main Paper}
For loss function:
\begin{equation}
	L_{similarity} = - \INNERDOT{\Q1, \Z2},
\end{equation}
we can get
\begin{align}
	\Q1 = \frac{q_1}{\|q_1\|} , \nonumber \\
	and \quad \frac{\partial L}{\partial \Q1} = - \Z2.
\end{align}

Here we use $p^{(i)}$ to index the $i^{th}$ element in vector $p$. Then we can calculate the gradients for $q_1$:
\begin{align}
	\frac{\partial \Q1^{(i)}}{\partial q_1^{(i)}} = \frac{1}{\|q_1\|} - q_1^{(i)} \frac{q_1^{(i)}}{\|q_1\|^3} = \frac{1}{\|q_1\|} (1 - \Q1^{(i)} \cdot \Q1^{(i)}),
\end{align}
\begin{align}
	\frac{\partial \Q1^{(j)}}{\partial q_1^{(i)}} = - q_1^{(j)} \frac{q_1^{(i)}}{\|q_1\|^3} = \frac{1}{\|q_1\|} (- \Q1^{(j)} \cdot \Q1^{(i)}),
\end{align}
\begin{align}
	\frac{\partial L}{\partial q_1^{(i)}} &= \sum_j \frac{\partial L}{\partial \Q1^{(j)}} \frac{\partial \Q1^{(j)}}{\partial q_1^{(i)}} \nonumber \\ 
	&= \frac{\partial L}{\partial \Q1^{(i)}} \frac{\partial \Q1^{(i)}}{\partial q_1^{(i)}} + \sum_{j!=i} \frac{\partial L}{\partial \Q1^{(j)}} \frac{\partial \Q1^{(j)}}{\partial q_1^{(i)}} \nonumber \\ 
	&= \frac{1}{\|q_1\|} (\frac{\partial L}{\partial \Q1^{(i)}} - \Q1^{(i)} (\sum_j \Q1^{(j)} \cdot \frac{\partial L}{\partial \Q1^{(j)}})) \nonumber \\
	&= \frac{1}{\|q_1\|} (\frac{\partial L}{\partial \Q1^{(i)}} - \Q1^{(i)} \cdot \langle \Q1, \frac{\partial L}{\partial \Q1}\rangle) \nonumber \\
	&= \frac{1}{\|q_1\|} (-\tilde{z}_2^{(i)} - \Q1^{(i)} \cdot \langle \Q1, -\tilde{z}_2\rangle).
\end{align}
Therefore, EQ. (9) in main paper is established. Based on the above equation, we can get
\begin{align}
	(\frac{\partial L}{\partial q_1^{(i)}})^2 =& \frac{1}{\|q_1\|^2} (\tilde{z}_2^{(i)^2} + \Q1^{(i)^2} \cdot \langle \Q1, -\tilde{z}_2\rangle ^ 2 \nonumber \\
	& - 2 (-\tilde{z}_2^{(i)} \cdot \Q1^{(i)}) \langle \Q1, -\tilde{z}_2\rangle),
\end{align}
\begin{align}
	\sum_i (\frac{\partial L}{\partial q_1^{(i)}})^2 =& \frac{1}{\|q_1\|^2} (\|-\tilde{z}_2\|^2 + \langle \Q1, -\tilde{z}_2\rangle ^ 2 - 2 \langle \Q1, -\tilde{z}_2\rangle ^ 2) \nonumber \\
	=& \frac{1}{\|q_1\|^2} (\|-\tilde{z}_2\|^2 - \langle \Q1, -\tilde{z}_2\rangle ^ 2) \nonumber \\
	=& \frac{1}{\|q_1\|^2} (1 - \langle \Q1, \tilde{z}_2\rangle ^ 2).
\end{align}
EQ. (10) in main paper is established as above.

\subsection*{Derivation of EQ. (11) and EQ. (12) in Main Paper}
For loss function:
\begin{align}
	L_{ce}=-\sum_i p(i|z_2, \tau)\log{p(i|q_1, \tau)}, \nonumber \\
	where \quad p(i|x, \tau) = \frac{\exp(\frac{x^{(i)}}{\tau})}{\sum_{j=1} \exp(\frac{x^{(j)}}{\tau})},
\end{align}
we can get another type of cross-entropy loss:
\begin{align}
	L_{ce} &= -\sum_i p(i|z_2, \tau)\log{\frac{\exp(\frac{q_1^{(i)}}{\tau})}{\sum_{j=1} \exp(\frac{q_1^{(j)}}{\tau})}} \nonumber \\
	&= -\sum_i p(i|z_2, \tau) (\frac{q_1^{(i)}}{\tau} - log{\sum_{j=1} \exp(\frac{q_1^{(j)}}{\tau})}) \nonumber \\
	&= \sum_i p(i|z_2, \tau) \cdot log{\sum_{j=1} \exp(\frac{q_1^{(j)}}{\tau})} - \sum_i p(i|z_2, \tau) \frac{q_1^{(i)}}{\tau} \nonumber \\
	&= log{\sum_{j=1} \exp(\frac{q_1^{(j)}}{\tau})} - \sum_i p(i|z_2, \tau) \frac{q_1^{(i)}}{\tau}.
\end{align}
This is the EQ. (12) in main paper. The gradients for $q_1$ can be calculated by
\begin{align}
	\frac{\partial{L_{ce}}}{\partial{q_1}^{(i)}} &= \frac{1}{\tau} \frac{exp(\frac{q_1^{(i)}}{\tau})}{\sum_{j=1} \exp(\frac{q_1^{(j)}}{\tau})} - \frac{1}{\tau} p(i|z_2, \tau) \nonumber \\
	&= \frac{1}{\tau} ( p(i|q_1, \tau) - p(i|z_2, \tau)).
\end{align}
Therefore, EQ. (11) in main paper is established.

\subsection*{Derivation of EQ. (14) in Main Paper}
For loss function
\begin{align}
	L_{mce} = - \sum_i p(i|z_2, \tau) \frac{\Q1^{(i)}}{\tau},
\end{align}
we have
\begin{align}
	\frac{\partial L_{mce}}{\partial \Q1^{(i)}} = - \frac{1}{\tau} p(i|z_2, \tau).
\end{align}
\begin{align}
	\frac{\partial L_{mce}}{\partial q_1^{(i)}} &= \sum_j \frac{\partial L_{mce}}{\partial \Q1^{(j)}} \frac{\partial \Q1^{(j)}}{\partial q_1^{(i)}} \nonumber \\ 
	&= \frac{1}{\|q_1\|} (\frac{\partial L_{mce}}{\partial \Q1^{(i)}} - \Q1^{(i)} \cdot \langle \Q1, \frac{\partial L_{mce}}{\partial \Q1}\rangle),
\end{align}
\begin{align}
	\|\frac{L_{mce}}{\partial q_1}\|^2 &= \frac{1}{\tau^2\|q_1\|^2} (\|\mathbf{P}(z_2)\|^2 - \langle \Q1, \mathbf{P}(z_2) \rangle ^ 2) \nonumber \\
	&= \frac{\|\mathbf{P}(z_2)\|^2}{\tau^2\|q_1\|^2} (1 - \langle \Q1, \frac{\mathbf{P}(z_2)}{\|\mathbf{P}(z_2)\|} \rangle ^ 2).
\end{align}

\subsection*{Explanation of $\tau_1$ in EQ. (15) in Main Paper}
For loss function
\begin{align}
	L_{iccl} =(\log{\sum_j \exp{\frac{\Q1^{(j)}}{\tau_{1}}}}) - (\sum_i p(i|z_2, \tau_2) \frac{\Q1^{(i)}}{\tau_{1}}),
\end{align}
we have
\begin{align}
	\frac{\partial L_{iccl}}{\partial \Q1^{(i)}} = \frac{1}{\tau_1} ( p(i|q_1, \tau_1) - p(i|z_2, \tau_2)),
\end{align}
\begin{align}
	\frac{\partial L_{iccl}}{\partial q_1^{(i)}} &= \sum_j \frac{\partial L_{iccl}}{\partial \Q1^{(j)}} \frac{\partial \Q1^{(j)}}{\partial q_1^{(i)}} \nonumber \\ 
	&= \frac{1}{\|q_1\|} (\frac{\partial L_{iccl}}{\partial \Q1^{(i)}} - \Q1^{(i)} \cdot \langle \Q1, \frac{\partial L_{iccl}}{\partial \Q1}\rangle),
\end{align}

\begin{align}
	\|\frac{L_{iccl}}{\partial q_1}\|^2 &= \frac{\|\mathbf{P}(q_1) - \mathbf{P}(z_2)\|^2}{\tau_1^2\|q_1\|^2} (1 - \langle \Q1, \frac{\mathbf{P}(q_1) - \mathbf{P}(z_2)}{\|\mathbf{P}(q_1) - \mathbf{P}(z_2)\|} \rangle ^ 2).
\end{align}
In main paper, we set $\tau_1 = \|\mathbf{P}(z_2)\|$. The value of $\tau_1$ is adaptive. However, as $\mathbf{P}(q_1)$ is closing to $\mathbf{P}(z_2)$, the magnitude of gradients may be vanishing. Therefore, we provide another setting for $\tau_1$. The default $\tau_1$ is a hyper-parameter (e.g., $\tau_1 = 0.1$ in DINO~\cite{dino}). To make $\tau_1$ become adaptive for different instances, we set $\tau_1$ to be $min(\tau_1, \|\mathbf{P}(z_2)\|)$.

\section*{Implementation Details}
The code has been open-sourced. Details can be seen in \textbf{Code/README.md}. We provide config files of many methods on ImageNet and Imagenette. It is convenient to reproduce the results for different methods. 

\subsection*{Initialization.}
For ResNet backbone, convolution layers' weights are initialized by HE initialization, and convolution layers' biases are initialized as 0. The fc layers' weight and bias for other components (e.g., projection, predictor) are initialized by xavier initialization~\cite{xavier}. The settings follow the details in BYOL.
\subsection*{Augmentation.}
During self-supervised training, we use the following image augmentations (PyTorch-like code).
\begin{itemize}
	\item $RandomResizedCrop$ with an area ratio uniformly sampled between 0.08 and 1.0, and an aspect ratio logarithmically sampled between 3/4 and 4/3.
	\item $Resize$ the patch to the target size of 224x224.
	\item $RandomHorizontalFlip$ the image with a probability of 0.5.
	\item $ColorJitter$ the \{brightness, contrast, saturation and hue\} of the image by the parameters \{0.4, 0.4, 0.4, 0.1\}. This augmentation operation is randomly applied with a probability of 0.8.
	\item $RandomGrayscale$ the image with a probability of 0.2.
	\item $GaussianBlur$ the image using the Gaussian kernel with std in [0.1, 2.0].  This augmentation operation is randomly applied with a probability of 1.0 and 0.1 for two independent transformations, respectively.
	\item $Solarization$ the image with a probability of 0.2 for one of the transformations.
	\item $ToTensor$. Scale the value of [0, 255] to [0.0, 1.0].
	\item $Normalize$ the image with estimated mean and std.
\end{itemize}

\subsection*{Optimizer and learning rate.}
For experiments on Imagenette, we use LARS with $lr$ = 2.0 for 1000 epochs, weight decay = 1e-6, momentum = 0.9, and batch size = 256. For those methods that use momentum encoder, the momentum value of momentum encoder is 0.996. These settings are shared in BYOL's github. For experiments on ImageNet, we use LARS with base $lr$ = 0.3, weight decay = 1e-6, momentum = 0.9, and batch size = 1024. The momentum value for momentum encoder is 0.99. For linear evaluation, the linear classifier training uses base lr = 0.3 with a cosine decay schedule for 100 epochs, weight decay = 1e-6, momentum = 0.9, batch size = 256 with SGD optimizer. We have also tried LARS optimizer with base $lr$ = 0.1, weight decay = 0, momentum = 0.9, epoch = 100, and batch size = 4096. This optimizer may give a competitive result.

\subsection*{Hyper-parameters for object detection.}
We use the detectron2 library for training the detection models and closely follow the evaluation settings from MoCo. We use Faster R-CNN C-4 detection model. The backbone is initialized by our pretrained model. For VOC 07 and VOC 07+12, training has 24K iterations using a batch size of 16 across 8 GPUs with SyncBatchNorm. The initial learning rate for the model is 0.1, which is reduced by a factor of 10 after 18K and 22K iterations. We use linear warmup for 1000 iterations. For COCO dataset, we train Mask R-CNN C-4 backbone on the COCO 2017 training set.  We use a learning
rate of 0.03 and keep the other parameters the same as in MoCo’s Github.
}

 
%

\bibliographystyle{IEEEtran}
\bibliography{IEEEabrv, bibfile}

\begin{thebibliography}{10}
\providecommand{\url}[1]{#1}
\csname url@samestyle\endcsname
\providecommand{\newblock}{\relax}
\providecommand{\bibinfo}[2]{#2}
\providecommand{\BIBentrySTDinterwordspacing}{\spaceskip=0pt\relax}
\providecommand{\BIBentryALTinterwordstretchfactor}{4}
\providecommand{\BIBentryALTinterwordspacing}{\spaceskip=\fontdimen2\font plus
\BIBentryALTinterwordstretchfactor\fontdimen3\font minus
  \fontdimen4\font\relax}
\providecommand{\BIBforeignlanguage}[2]{{%
\expandafter\ifx\csname l@#1\endcsname\relax
\typeout{** WARNING: IEEEtran.bst: No hyphenation pattern has been}%
\typeout{** loaded for the language `#1'. Using the pattern for}%
\typeout{** the default language instead.}%
\else
\language=\csname l@#1\endcsname
\fi
#2}}
\providecommand{\BIBdecl}{\relax}
\BIBdecl

\bibitem{hjelm2018learning}
R.~D. Hjelm, A.~Fedorov, S.~Lavoie-Marchildon, K.~Grewal, P.~Bachman,
  A.~Trischler, and Y.~Bengio, ``Learning deep representations by mutual
  information estimation and maximization,'' in \emph{International Conference
  on Learning Representations}, 2018.

\bibitem{oord2018representation}
A.~v.~d. Oord, Y.~Li, and O.~Vinyals, ``Representation learning with
  contrastive predictive coding,'' \emph{arXiv preprint arXiv:1807.03748},
  2018.

\bibitem{bachman2019learning}
P.~Bachman \emph{et~al.}, ``Learning representations by maximizing mutual
  information across views,'' in \emph{Advances in Neural Information
  Processing Systems (NeurIPS)}, 2019, pp. 15\,535--15\,545.

\bibitem{caron2019unsupervised}
M.~Caron, P.~Bojanowski, J.~Mairal, and A.~Joulin, ``Unsupervised pre-training
  of image features on non-curated data,'' in \emph{Proceedings of the IEEE/CVF
  International Conference on Computer Vision}, 2019, pp. 2959--2968.

\bibitem{S4L}
X.~Zhai, A.~Oliver, A.~Kolesnikov, and L.~Beyer, ``S4l: Self-supervised
  semi-supervised learning,'' in \emph{Proceedings of the IEEE/CVF
  International Conference on Computer Vision}, 2019, pp. 1476--1485.

\bibitem{MoCo}
K.~He, H.~Fan, Y.~Wu, S.~Xie, and R.~Girshick, ``Momentum contrast for
  unsupervised visual representation learning,'' in \emph{Proceedings of the
  IEEE/CVF Conference on Computer Vision and Pattern Recognition}, 2020, pp.
  9729--9738.

\bibitem{gpt3}
T.~B. Brown, B.~Mann, N.~Ryder, M.~Subbiah, J.~Kaplan, P.~Dhariwal,
  A.~Neelakantan, P.~Shyam, G.~Sastry, A.~Askell, S.~Agarwal,
  A.~Herbert{-}Voss, G.~Krueger, T.~Henighan, R.~Child, A.~Ramesh, D.~M.
  Ziegler, J.~Wu, C.~Winter, C.~Hesse, M.~Chen, E.~Sigler, M.~Litwin, S.~Gray,
  B.~Chess, J.~Clark, C.~Berner, S.~McCandlish, A.~Radford, I.~Sutskever, and
  D.~Amodei, ``Language models are few-shot learners,'' in \emph{Advances in
  Neural Information Processing Systems (NeurIPS)}, H.~Larochelle, M.~Ranzato,
  R.~Hadsell, M.~Balcan, and H.~Lin, Eds., 2020.

\bibitem{SimCLR}
T.~Chen, S.~Kornblith, M.~Norouzi, and G.~E. Hinton, ``A simple framework for
  contrastive learning of visual representations,'' in \emph{Proceedings of the
  37th International Conference on Machine Learning, {ICML} 2020}, vol. 119,
  2020, pp. 1597--1607.

\bibitem{henaff2020data}
O.~Henaff, ``Data-efficient image recognition with contrastive predictive
  coding,'' in \emph{International Conference on Machine Learning}.\hskip 1em
  plus 0.5em minus 0.4em\relax PMLR, 2020, pp. 4182--4192.

\bibitem{barlowtwins}
\BIBentryALTinterwordspacing
J.~Zbontar, L.~Jing, I.~Misra, Y.~LeCun, and S.~Deny, ``Barlow twins:
  Self-supervised learning via redundancy reduction,'' in \emph{Proceedings of
  the 38th International Conference on Machine Learning, {ICML} 2021, 18-24
  July 2021, Virtual Event}, ser. Proceedings of Machine Learning Research,
  M.~Meila and T.~Zhang, Eds., vol. 139.\hskip 1em plus 0.5em minus 0.4em\relax
  {PMLR}, 2021, pp. 12\,310--12\,320. [Online]. Available:
  \url{http://proceedings.mlr.press/v139/zbontar21a.html}
\BIBentrySTDinterwordspacing

\bibitem{dino}
M.~Caron, H.~Touvron, I.~Misra, H.~J{\'e}gou, J.~Mairal, P.~Bojanowski, and
  A.~Joulin, ``Emerging properties in self-supervised vision transformers,''
  \emph{arXiv preprint arXiv:2104.14294}, 2021.

\bibitem{MAE}
\BIBentryALTinterwordspacing
K.~He, X.~Chen, S.~Xie, Y.~Li, P.~Doll{\'{a}}r, and R.~B. Girshick, ``Masked
  autoencoders are scalable vision learners,'' \emph{CoRR}, vol.
  abs/2111.06377, 2021. [Online]. Available:
  \url{https://arxiv.org/abs/2111.06377}
\BIBentrySTDinterwordspacing

\bibitem{bojanowski2017unsupervised}
P.~Bojanowski and A.~Joulin, ``Unsupervised learning by predicting noise,'' in
  \emph{International Conference on Machine Learning}.\hskip 1em plus 0.5em
  minus 0.4em\relax PMLR, 2017, pp. 517--526.

\bibitem{dosovitskiy2015discriminative}
A.~Dosovitskiy, P.~Fischer, J.~T. Springenberg, M.~Riedmiller, and T.~Brox,
  ``Discriminative unsupervised feature learning with exemplar convolutional
  neural networks,'' \emph{IEEE Transactions on Pattern Analysis and Machine
  Intelligence}, vol.~38, no.~9, pp. 1734--1747, 2015.

\bibitem{wu2018unsupervised}
Z.~Wu, Y.~Xiong, S.~X. Yu, and D.~Lin, ``Unsupervised feature learning via
  non-parametric instance discrimination,'' in \emph{Proceedings of the IEEE
  Conference on Computer Vision and Pattern Recognition}, 2018, pp. 3733--3742.

\bibitem{cao2020parametric}
Y.~Cao, Z.~Xie, B.~Liu, Y.~Lin, Z.~Zhang, and H.~Hu, ``Parametric instance
  classification for unsupervised visual feature learning,'' in \emph{Advances
  in Neural Information Processing Systems (NeurIPS)}, H.~Larochelle,
  M.~Ranzato, R.~Hadsell, M.~Balcan, and H.~Lin, Eds., 2020.

\bibitem{DBLP:conf/nips/Tian0PKSI20}
Y.~Tian, C.~Sun, B.~Poole, D.~Krishnan, C.~Schmid, and P.~Isola, ``What makes
  for good views for contrastive learning?'' in \emph{Advances in Neural
  Information Processing Systems (NeurIPS)}, 2020.

\bibitem{BYOL}
J.~Grill, F.~Strub, F.~Altch{\'{e}}, C.~Tallec, P.~H. Richemond,
  E.~Buchatskaya, C.~Doersch, B.~{\'{A}}. Pires, Z.~Guo, M.~G. Azar, B.~Piot,
  K.~Kavukcuoglu, R.~Munos, and M.~Valko, ``Bootstrap your own latent - {A} new
  approach to self-supervised learning,'' in \emph{Advances in Neural
  Information Processing Systems (NeurIPS)}, 2020.

\bibitem{SimSiam}
X.~Chen and K.~He, ``Exploring simple siamese representation learning,'' in
  \emph{{IEEE} Conference on Computer Vision and Pattern Recognition, {CVPR}
  2021, virtual, June 19-25, 2021}.\hskip 1em plus 0.5em minus 0.4em\relax
  Computer Vision Foundation / {IEEE}, 2021, pp. 15\,750--15\,758.

\bibitem{richemond2020byol}
P.~H. Richemond, J.-B. Grill, F.~Altch{\'e}, C.~Tallec, F.~Strub, A.~Brock,
  S.~Smith, S.~De, R.~Pascanu, B.~Piot \emph{et~al.}, ``Byol works even without
  batch statistics,'' \emph{arXiv preprint arXiv:2010.10241}, 2020.

\bibitem{tian2020understanding}
Y.~Tian, L.~Yu, X.~Chen, and S.~Ganguli, ``Understanding self-supervised
  learning with dual deep networks,'' \emph{arXiv preprint arXiv:2010.00578},
  2020.

\bibitem{directpred}
\BIBentryALTinterwordspacing
Y.~Tian, X.~Chen, and S.~Ganguli, ``Understanding self-supervised learning
  dynamics without contrastive pairs,'' in \emph{Proceedings of the 38th
  International Conference on Machine Learning, {ICML} 2021, 18-24 July 2021,
  Virtual Event}, ser. Proceedings of Machine Learning Research, M.~Meila and
  T.~Zhang, Eds., vol. 139.\hskip 1em plus 0.5em minus 0.4em\relax {PMLR},
  2021, pp. 10\,268--10\,278. [Online]. Available:
  \url{http://proceedings.mlr.press/v139/tian21a.html}
\BIBentrySTDinterwordspacing

\bibitem{MoCov3}
X.~Chen, S.~Xie, and K.~He, ``An empirical study of training self-supervised
  vision transformers,'' \emph{arXiv preprint arXiv:2104.02057}, 2021.

\bibitem{PIRL}
I.~Misra and L.~v.~d. Maaten, ``Self-supervised learning of pretext-invariant
  representations,'' in \emph{Proceedings of the IEEE/CVF Conference on
  Computer Vision and Pattern Recognition}, 2020, pp. 6707--6717.

\bibitem{yang2016joint}
J.~Yang, D.~Parikh, and D.~Batra, ``Joint unsupervised learning of deep
  representations and image clusters,'' in \emph{Proceedings of the IEEE
  Conference on Computer Vision and Pattern Recognition}, 2016, pp. 5147--5156.

\bibitem{xie2016unsupervised}
\BIBentryALTinterwordspacing
J.~Xie, R.~B. Girshick, and A.~Farhadi, ``Unsupervised deep embedding for
  clustering analysis,'' in \emph{Proceedings of the 33nd International
  Conference on Machine Learning, {ICML} 2016, New York City, NY, USA, June
  19-24, 2016}, ser. {JMLR} Workshop and Conference Proceedings, M.~Balcan and
  K.~Q. Weinberger, Eds., vol.~48.\hskip 1em plus 0.5em minus 0.4em\relax
  JMLR.org, 2016, pp. 478--487. [Online]. Available:
  \url{http://proceedings.mlr.press/v48/xieb16.html}
\BIBentrySTDinterwordspacing

\bibitem{yan2020clusterfit}
X.~Yan, I.~Misra, A.~Gupta, D.~Ghadiyaram, and D.~Mahajan, ``Clusterfit:
  Improving generalization of visual representations,'' in \emph{Proceedings of
  the IEEE/CVF Conference on Computer Vision and Pattern Recognition}, 2020,
  pp. 6509--6518.

\bibitem{huang2019unsupervised}
J.~Huang, Q.~Dong, S.~Gong, and X.~Zhu, ``Unsupervised deep learning by
  neighbourhood discovery,'' in \emph{International Conference on Machine
  Learning}.\hskip 1em plus 0.5em minus 0.4em\relax PMLR, 2019, pp. 2849--2858.

\bibitem{caron2018deep}
M.~Caron, P.~Bojanowski, A.~Joulin, and M.~Douze, ``Deep clustering for
  unsupervised learning of visual features,'' in \emph{Proceedings of the
  European Conference on Computer Vision (ECCV)}, 2018, pp. 132--149.

\bibitem{DBLP:conf/iclr/AsanoRV20a}
Y.~M. Asano \emph{et~al.}, ``Self-labelling via simultaneous clustering and
  representation learning,'' in \emph{8th International Conference on Learning
  Representations, {ICLR} 2020}, 2020.

\bibitem{SwAV}
M.~Caron, I.~Misra, J.~Mairal, P.~Goyal, P.~Bojanowski, and A.~Joulin,
  ``Unsupervised learning of visual features by contrasting cluster
  assignments,'' in \emph{Advances in Neural Information Processing Systems
  (NeurIPS)}, 2020.

\bibitem{li2021contrastive}
Y.~Li, P.~Hu, Z.~Liu, D.~Peng, J.~T. Zhou, and X.~Peng, ``Contrastive
  clustering,'' in \emph{2021 AAAI Conference on Artificial Intelligence
  (AAAI)}, 2021.

\bibitem{zhao2020distilling}
\BIBentryALTinterwordspacing
N.~Zhao, Z.~Wu, R.~W.~H. Lau, and S.~Lin, ``Distilling localization for
  self-supervised representation learning,'' in \emph{Thirty-Fifth {AAAI}
  Conference on Artificial Intelligence, {AAAI} 2021, Thirty-Third Conference
  on Innovative Applications of Artificial Intelligence, {IAAI} 2021, The
  Eleventh Symposium on Educational Advances in Artificial Intelligence, {EAAI}
  2021, Virtual Event, February 2-9, 2021}.\hskip 1em plus 0.5em minus
  0.4em\relax {AAAI} Press, 2021, pp. 10\,990--10\,998. [Online]. Available:
  \url{https://ojs.aaai.org/index.php/AAAI/article/view/17312}
\BIBentrySTDinterwordspacing

\bibitem{cuturi2013sinkhorn}
M.~Cuturi, ``Sinkhorn distances: Lightspeed computation of optimal transport,''
  \emph{Advances in Neural Information Processing Systems (NeurIPS)}, vol.~26,
  pp. 2292--2300, 2013.

\bibitem{hadsell2006dimensionality}
R.~Hadsell, S.~Chopra, and Y.~LeCun, ``Dimensionality reduction by learning an
  invariant mapping,'' in \emph{2006 IEEE Computer Society Conference on
  Computer Vision and Pattern Recognition (CVPR'06)}, vol.~2.\hskip 1em plus
  0.5em minus 0.4em\relax IEEE, 2006, pp. 1735--1742.

\bibitem{DBLP:conf/icml/0001I20}
T.~Wang and P.~Isola, ``Understanding contrastive representation learning
  through alignment and uniformity on the hypersphere,'' in \emph{Proceedings
  of the 37th International Conference on Machine Learning, {ICML} 2020, 13-18
  July 2020, Virtual Event}, ser. Proceedings of Machine Learning Research,
  vol. 119, 2020, pp. 9929--9939.

\bibitem{chen2020big}
T.~Chen, S.~Kornblith, K.~Swersky, M.~Norouzi, and G.~E. Hinton, ``Big
  self-supervised models are strong semi-supervised learners,'' in
  \emph{Advances in Neural Information Processing Systems (NeurIPS)},
  H.~Larochelle, M.~Ranzato, R.~Hadsell, M.~Balcan, and H.~Lin, Eds., 2020.

\bibitem{hinton2015distilling}
G.~Hinton, O.~Vinyals, and J.~Dean, ``Distilling the knowledge in a neural
  network,'' \emph{arXiv preprint arXiv:1503.02531}, 2015.

\bibitem{tsne}
L.~Van~der Maaten and G.~Hinton, ``Visualizing data using t-sne.''
  \emph{Journal of machine learning research}, vol.~9, no.~11, 2008.

\bibitem{imagenette}
F.~J. Howard, ``The imagenette dataset,'' https:
  //github.com/fastai/imagenette., 2020.

\bibitem{he2015delving}
K.~He, X.~Zhang, S.~Ren, and J.~Sun, ``Delving deep into rectifiers: Surpassing
  human-level performance on imagenet classification,'' in \emph{Proceedings of
  the IEEE International Conference on Computer Vision}, 2015, pp. 1026--1034.

\bibitem{loshchilov2016sgdr}
\BIBentryALTinterwordspacing
I.~Loshchilov and F.~Hutter, ``{SGDR:} stochastic gradient descent with warm
  restarts,'' in \emph{5th International Conference on Learning
  Representations, {ICLR} 2017, Toulon, France, April 24-26, 2017, Conference
  Track Proceedings}.\hskip 1em plus 0.5em minus 0.4em\relax OpenReview.net,
  2017. [Online]. Available: \url{https://openreview.net/forum?id=Skq89Scxx}
\BIBentrySTDinterwordspacing

\bibitem{he2016identity}
K.~He, X.~Zhang, S.~Ren, and J.~Sun, ``Identity mappings in deep residual
  networks,'' in \emph{European Conference on Computer Vision}.\hskip 1em plus
  0.5em minus 0.4em\relax Springer, 2016, pp. 630--645.

\bibitem{goyal2017accurate}
P.~Goyal, P.~Doll{\'a}r, R.~Girshick, P.~Noordhuis, L.~Wesolowski, A.~Kyrola,
  A.~Tulloch, Y.~Jia, and K.~He, ``Accurate, large minibatch sgd: Training
  imagenet in 1 hour,'' \emph{arXiv preprint arXiv:1706.02677}, 2017.

\bibitem{DBLP:conf/icml/IoffeS15}
S.~Ioffe and C.~Szegedy, ``Batch normalization: Accelerating deep network
  training by reducing internal covariate shift,'' in \emph{Proceedings of the
  32nd International Conference on Machine Learning, {ICML} 2015, Lille,
  France, 6-11 July 2015}, 2015.

\bibitem{you2017large}
Y.~You, I.~Gitman, and B.~Ginsburg, ``Large batch training of convolutional
  networks,'' \emph{arXiv preprint arXiv:1708.03888}, 2017.

\bibitem{krizhevsky2012imagenet}
A.~Krizhevsky, I.~Sutskever, and G.~E. Hinton, ``Imagenet classification with
  deep convolutional neural networks,'' \emph{Advances in Neural Information
  Processing Systems (NeurIPS)}, vol.~25, pp. 1097--1105, 2012.

\bibitem{mocov2}
X.~Chen, H.~Fan, R.~Girshick, and K.~He, ``Improved baselines with momentum
  contrastive learning,'' \emph{arXiv preprint arXiv:2003.04297}, 2020.

\bibitem{ren2015faster}
S.~Ren, K.~He, R.~Girshick, and J.~Sun, ``Faster r-cnn: Towards real-time
  object detection with region proposal networks,'' \emph{Advances in Neural
  Information Processing Systems (NeurIPS)}, vol.~28, pp. 91--99, 2015.

\bibitem{dosovitskiy2020image}
\BIBentryALTinterwordspacing
A.~Dosovitskiy, L.~Beyer, A.~Kolesnikov, D.~Weissenborn, X.~Zhai,
  T.~Unterthiner, M.~Dehghani, M.~Minderer, G.~Heigold, S.~Gelly, J.~Uszkoreit,
  and N.~Houlsby, ``An image is worth 16x16 words: Transformers for image
  recognition at scale,'' in \emph{9th International Conference on Learning
  Representations, {ICLR} 2021, Virtual Event, Austria, May 3-7, 2021}.\hskip
  1em plus 0.5em minus 0.4em\relax OpenReview.net, 2021. [Online]. Available:
  \url{https://openreview.net/forum?id=YicbFdNTTy}
\BIBentrySTDinterwordspacing

\bibitem{coco}
T.-Y. Lin, M.~Maire, S.~Belongie, J.~Hays, P.~Perona, D.~Ramanan,
  P.~Doll{\'a}r, and C.~L. Zitnick, ``Microsoft coco: Common objects in
  context,'' in \emph{European conference on computer vision}.\hskip 1em plus
  0.5em minus 0.4em\relax Springer, 2014, pp. 740--755.

\bibitem{maskrcnn}
\BIBentryALTinterwordspacing
K.~He, G.~Gkioxari, P.~Doll{\'{a}}r, and R.~B. Girshick, ``Mask {R-CNN},''
  \emph{CoRR}, vol. abs/1703.06870, 2017. [Online]. Available:
  \url{http://arxiv.org/abs/1703.06870}
\BIBentrySTDinterwordspacing

\bibitem{TransferLearning}
P.~Goyal, D.~Mahajan, A.~Gupta, and I.~Misra, ``Scaling and benchmarking
  self-supervised visual representation learning,'' in \emph{Proceedings of the
  IEEE/CVF International Conference on Computer Vision}, 2019, pp. 6391--6400.

\bibitem{VOC}
M.~Everingham, L.~Van~Gool, C.~K. Williams, J.~Winn, and A.~Zisserman, ``The
  pascal visual object classes (voc) challenge,'' \emph{International Journal
  of Computer Vision}, vol.~88, no.~2, pp. 303--338, 2010.

\bibitem{wu2019detectron2}
Y.~Wu, A.~Kirillov, F.~Massa, W.-Y. Lo, and R.~Girshick, ``Detectron2,''
  \url{https://github.com/facebookresearch/detectron2}, 2019.

\bibitem{ericsson2021well}
L.~Ericsson, H.~Gouk, and T.~M. Hospedales, ``How well do self-supervised
  models transfer?'' in \emph{Proceedings of the IEEE/CVF Conference on
  Computer Vision and Pattern Recognition}, 2021, pp. 5414--5423.

\bibitem{xavier}
\BIBentryALTinterwordspacing
X.~Glorot and Y.~Bengio, ``Understanding the difficulty of training deep
  feedforward neural networks,'' in \emph{Proceedings of the Thirteenth
  International Conference on Artificial Intelligence and Statistics, {AISTATS}
  2010, Chia Laguna Resort, Sardinia, Italy, May 13-15, 2010}, ser. {JMLR}
  Proceedings, Y.~W. Teh and D.~M. Titterington, Eds., vol.~9.\hskip 1em plus
  0.5em minus 0.4em\relax JMLR.org, 2010, pp. 249--256. [Online]. Available:
  \url{http://proceedings.mlr.press/v9/glorot10a.html}
\BIBentrySTDinterwordspacing

\end{thebibliography}

\newpage

\section{Biography Section}
\vspace{-33pt}
\begin{IEEEbiography}[{\includegraphics[width=1in,height=1.25in,clip,keepaspectratio]{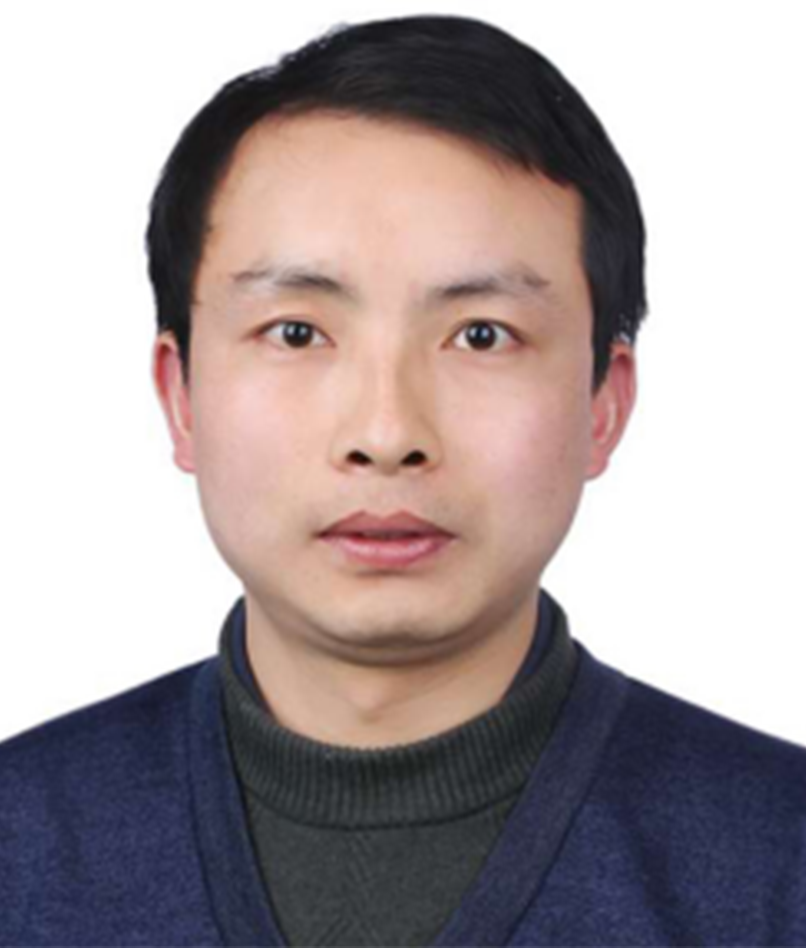}}]{Jidong Ge}
received the Ph.D. degree in computer science from Nanjing University in 2007. He is currently an Associate Professor with the Software Institute, Nanjing University. His current research interests include machine learning, deep learning, image processing, software engineering, NLP, and process mining. His research results have been published in more than 100 papers in international journals and conference papers, including IEEE TRANSACTIONS ON SOFTWARE ENGINEERING (TSE), IEEE/ACM TRANSACTIONS ON NETWORKING (TNET), IEEE TRANSACTIONS ON PARALLEL AND DISTRIBUTED SYSTEMS (TPDS), IEEE TRANSACTIONS ON MOBILE COMPUTING (TMC), IEEE TRANSACTIONS ON SERVICES COMPUTING (TSC), ACM Transactions on Knowledge Discovery from Data (TKDD), IEEE/ACM TRANSACTIONS ON AUDIO, SPEECH, AND LANGUAGE PROCESSING (TASLP), Automated Software Engineering (An International Journal), Computer Networks, Journal of Parallel and Distributed Computing (JPDC), Future Generation Computer Systems (FGCS), Journal of Systems and Software (JSS), Information Sciences, Journal of Network and Computer Applications (JNCA), Journal of Software: Evolution and Process (JSEP), Expert Systems with Applications (ESA), AAAI, ICSE, ESEC/FSE, ASE, IWQoS, and GlobeCom. 
\end{IEEEbiography}

\vspace{11pt}
\vspace{-33pt}
\begin{IEEEbiography}[{\includegraphics[width=1in,height=1.25in,clip,keepaspectratio]{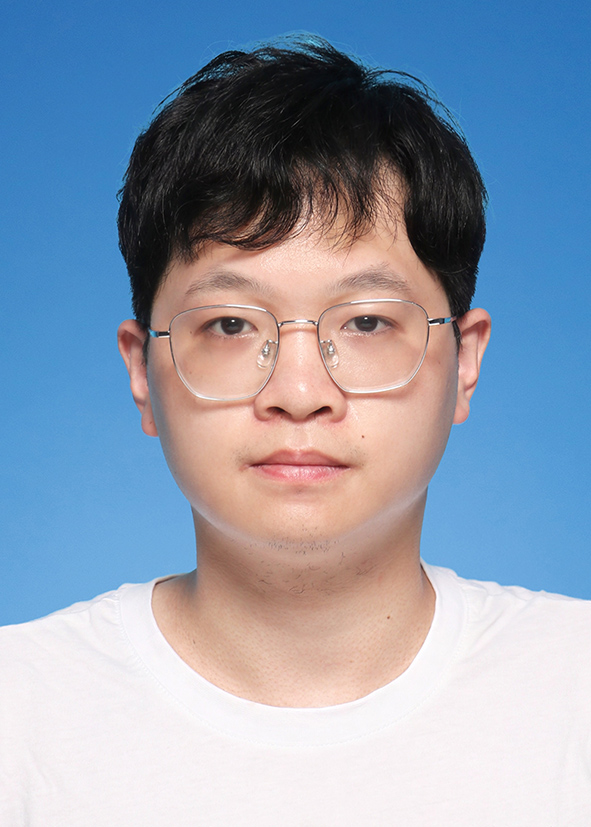}}]{Yuxiang Liu}
	received the BS degree fromthe Software Institute, Nanjing University. 
	He is currently working toward the MS degree in the Software Institute, 
	Nanjing University. His research interests include computer vision and deep learning.
\end{IEEEbiography}


\vspace{11pt}
\vspace{-33pt}
\begin{IEEEbiography}[{\includegraphics[width=1in,height=1.25in,clip,keepaspectratio]{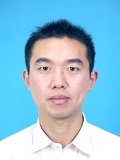}}]{Jie Gui}
	(SM'16) is currently a Professor in the School of Cyber Science and Engineering, Southeast University. He received a B.S. degree in Computer Science from Hohai University, Nanjing, China, in 2004, an M.S. degree in Computer Applied Technology from Hefei Institutes of Physical Science, Chinese Academy of Sciences, Hefei, China, in 2007, and a Ph.D. degree in Pattern Recognition and Intelligent Systems from the University of Science and Technology of China, Hefei, China, in 2010. He is the Associate Editor of Neurocomputing, a Senior Member of the IEEE and ACM, and a CCF Distinguished Member. He has published more than 60 papers in international journals and conferences such as IEEE TPAMI, IEEE TNNLS, IEEE TCYB, IEEE TIP, IEEE TCSVT, IEEE TSMCS, KDD, AAAI, and ACM MM. He is the Area Chair, Senior PC Member, or PC Member of many conferences, such as NeurIPS and ICML. His research interests include machine learning, pattern recognition, and image processing.
\end{IEEEbiography}

\vspace{11pt}
\vspace{-33pt}
\begin{IEEEbiography}[{\includegraphics[width=1in,height=1.25in,clip,keepaspectratio]{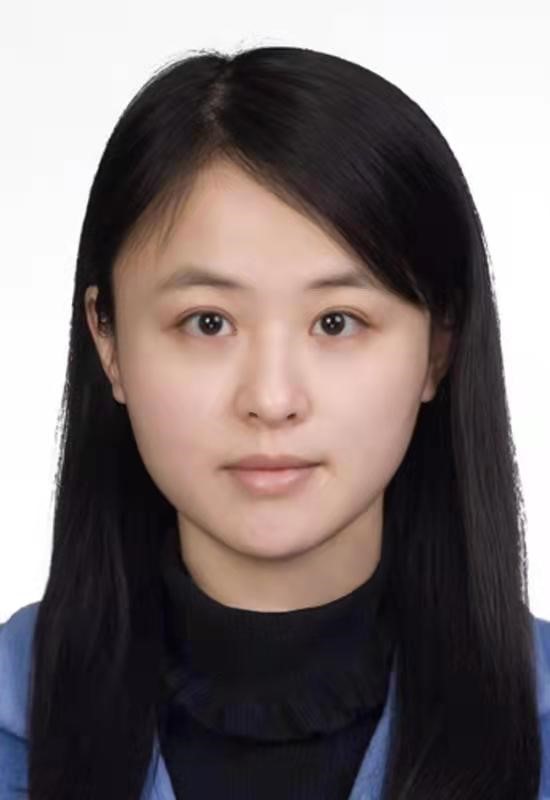}}]{Lanting Fang} is a Lecturer with the School of Cyber Science and Engineering, Southeast University, Nanjing, China. She received her Ph.D. in 2018 from Southeast University, Nanjing, China. Her current research interests include NLP, information extraction, social media and sentiment analysis.
\end{IEEEbiography}

\vspace{11pt}
\vspace{-33pt}
\begin{IEEEbiography}[{\includegraphics[width=1in,height=1.25in,clip,keepaspectratio]{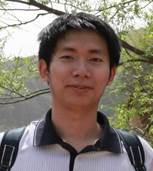}}]{Ming Lin} is a Senior Applied Scientist at Amazon.com LCC. His research interests include Mathematical Foundation of Deep Learning and Statistical Machine Learning, with their applications in deep learning acceleration, computer vision and mobile AI. Before he joined Amazon, he was a Staff Algorithm Engineer at DAMO Academy of Alibaba Group (U.S.) from April 2018 to July 2022. He was a Research Investigator in the Medical School of Michigan University from Sep 2015 to April 2018. He worked as a postdoctoral researcher in the School of Computer Science at Carnegie Mellon University from July 2014 to Sep 2015. He received his Ph.D. degree in computer science from Tsinghua University in 2014. During his Ph.D. study, he had been a visiting scholar in Michigan State University and in CMU from Dec 2013 to July 2014.
\end{IEEEbiography}

\vspace{11pt}
\vspace{-33pt}
\begin{IEEEbiography}[{\includegraphics[width=1in,height=1.25in,clip,keepaspectratio]{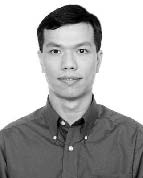}}]{James Tin-Yau Kwok}  (Fellow, IEEE) received
the Ph.D. degree in computer science from The
Hong Kong University of Science and Technology,
Hong Kong, in 1996. He is currently a Professor
with the Department of Computer Science
and Engineering, The Hong Kong University of
Science and Technology. His current research
interests include kernel methods, machine learning,
pattern recognition, and artificial neural networks.
He received the IEEE Outstanding Paper
Award in 2004 and the Second Class Award in
Natural Sciences from the Ministry of Education, China, in 2008. He has
been a Program Co-Chair for a number of international conferences,
and served as an Associate Editor for the IEEE TRANS-ACTIONS ON
NEURAL NETWORKS AND LEARNING SYSTEMS from 2006 to 2012.
He is currently an Associate Editor of Neurocomputing.
\end{IEEEbiography}

\vspace{11pt}
\vspace{-33pt}
\begin{IEEEbiography}[{\includegraphics[width=1in,height=1.25in,clip,keepaspectratio]{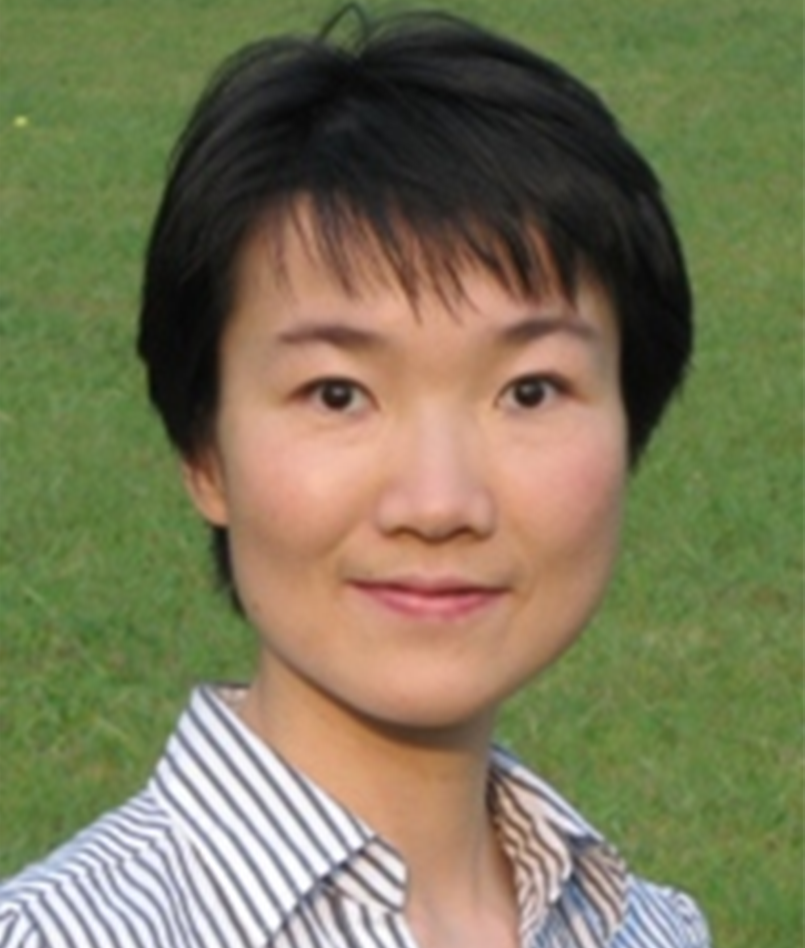}}]{LiGuo Huang}
    is an Associate Professor, Department of Computer Science and Engineering at the Southern Methodist University (SMU), Dallas, TX, USA. She received both her Ph.D. (2006) and M.S. from the Computer Science Department and Center for Systems and Software Engineering (CSSE) at the University of Southern California (USC). Her current research centers around process modeling, simulation and improvement, systems and software quality assurance and information dependability, mining systems and software engineering repository, stakeholder/value-based software engineering, and software metrics. Her research has been supported by NSF, U.S. Department of Defense (DoD) and NSA. She had been intensively involved in initiating the research on stakeholder/value-based integration of systems and software engineering.
\end{IEEEbiography}

\vspace{11pt}
\vspace{-33pt}
\begin{IEEEbiography}[{\includegraphics[width=1in,height=1.25in,clip,keepaspectratio]{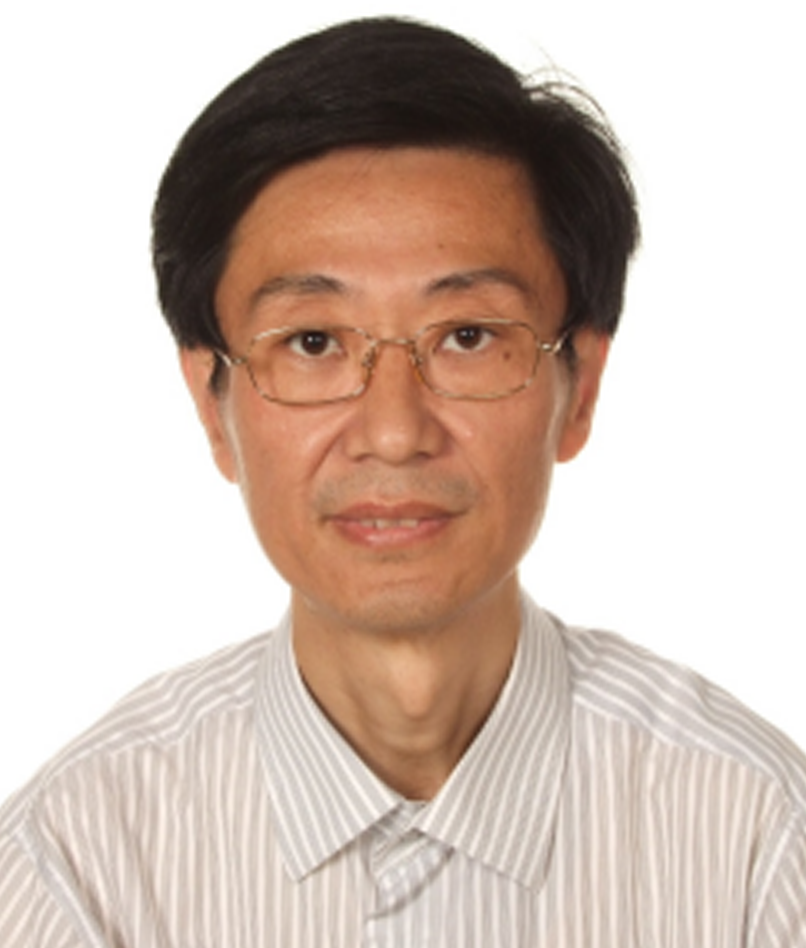}}]{Bin Luo}
	 is a full Professor at Software Institute, Nanjing University. His main research interests include machine learning, deep learning, image processing, software engineering, NLP, process mining. His research results have been published in more than 100 papers in international journals and conference proceedings including IEEE TPDS, IEEE TMC, IEEE TSC, IEEE/ACM TASLP, ACM TKDD, ACM TOIST, FGCS, JSS, Inf. Sci., ESA, ExpSys etc. He is leading the institute of applied software engineering at Nanjing University.
\end{IEEEbiography}

\vfill

\end{document}